\documentclass[10pt,twocolumn,a4]{article}

\usepackage{scrextend}
\usepackage{booktabs}
\usepackage{iccv}
\usepackage{times}
\usepackage{graphicx}
\usepackage{amsmath,amsthm}
\usepackage{amssymb}
\usepackage{graphicx}
\usepackage{subfig}
\usepackage{color}
\usepackage{algorithm,algorithmic}
\usepackage{tikz,pgfplots}
 \usepackage{import}

\newtheorem{theorem}{Theorem}[section]

\newtheorem{definition}[theorem]{Definition}
\newtheorem{proposition}[theorem]{Proposition}

\long\def\ignore#1{}

\usepackage[pagebackref=true,breaklinks=true,letterpaper=true,colorlinks,bookmarks=false]{hyperref}

\def\x{{\mathbf{x}}}

\def\tree{{\mathcal{T}}}

\def\X{{\bf X}}
\def\C{{\mathcal{C}}}
\def\V{{\mathcal{V}}}
\def\E{{\mathcal{E}}}

\def\calC{{\cal C}}

\def\calE{{\cal E}}
\def\calF{{\cal F}}

\def\calL{{\cal L}}

\def\calV{{\cal V}}

\newcommand{\bx}{\mbox{\boldmath $x$}}

\newcommand{\bX}{\mbox{\boldmath $X$}}

\def\myparagraph#1{\vspace{2pt}\noindent{\bf #1~~}}

\iccvfinalcopy 

\def\iccvPaperID{333} 

\ificcvfinal\fi

\begin{document}

\title{\vspace{-10mm} Simplifying  Energy Optimization using Partial Enumeration   \vspace{-3mm}}
 
\author{Carl Olsson     \hspace{8ex}     Johannes  Ul\'en   \\
Centre for Mathematical Sciences\\
Lund University, Sweden\\
{\tt\small calle@maths.lth.se    \hspace{0.1ex}    ulen@maths.lth.se}
\and
Yuri Boykov \\
Computer Science\\
UWO, Canada\\
{\tt \small yuri@csd.uwo.ca}
\and
Vladimir Kolmogorov \\
Inst. of Science \& Technology \\ Austria\\
{\tt \small vnk@ist.ac.at}
}

\maketitle

\begin{abstract} \vspace{-3mm}
Energies with high-order non-submodular interactions 
have been shown to be very useful in vision due to their high
modeling power. Optimization of such energies, however,
is generally NP-hard. A naive approach that works for
small problem instances is exhaustive search, that is, enumeration of all possible labelings of the underlying graph.
We propose a general minimization approach for large graphs based on enumeration of labelings of certain small patches.
This {\em partial enumeration} technique reduces complex high-order energy formulations to pairwise 
{\em Constraint Satisfaction Problems} with unary costs (uCSP), which can be efficiently solved using standard methods like TRW-S.
Our approach outperforms a number of existing state-of-the-art algorithms on well known difficult problems 
(e.g. curvature regularization, stereo, deconvolution); it gives near global minimum and better speed.

Our main application of interest is curvature regularization.
In the context of segmentation, our partial enumeration technique allows to evaluate curvature directly on small patches
using a novel {\em integral geometry} approach.
\footnote{
This work has been funded by the Swedish Research Council (grant 2012-4213), the Crafoord Foundation, 
the Canadian Foundation for Innovation (CFI 10318) and the Canadian NSERC Discovery Program (grant 298299-2012RGPIN).
We would also like to thank Prof. Olga Veksler for referring to {\em partical enumeration} as a ``cute idea''.
}


\end{abstract}
\vspace{-3mm}
\section{Introduction}

Optimization of curvature and higher-order regularizers, in general,  has significant potential in segmentation, 
stereo, 3D reconstruction, image restoration, in-painting, and other applications. It is widely known as
a challenging problem with a long history of research in computer vision. For example, when Geman and Geman 
introduced MRF models to computer vision \cite{geman-geman-pami-1984} they proposed first- and second-order 
regularization based on {\em line process}. The popular {\em active contours} framework \cite{Kass:88} uses elastic (first-order) 
and bending (second-order) energies for segmentation. Dynamic programming was used for curvature-based inpainting \cite{MasnouMorel:98}. Curvature was also studied within PDE-based  \cite{ChanShen:01} 
and {\em level-sets} \cite{DroskeRumpf:04} approaches to image analysis. 

Recently there has been a revival of interest in second-order smoothness for discrete MRF settings.
Due to the success of global optimization methods for first-order MRF models \cite{boykov-etal-pami-2001,ishikawa-pami-2003}
researchers now focus on more difficult second-order functionals \cite{woodford2009} including various discrete 
approximations of curvature \cite{schoenemann-etal-ijcv-2012,elzehiry-grady-cvpr-2010,strandmark-kahl-emmcvpr-2011}. 
Similarly, recent progress on global optimization techniques for first-order continuous geometric functionals \cite{Nikolova:SIAM06,Pock:SIAM10,Lellmann:SIAM11,yuan:eccv10} has lead to extensions for curvature \cite{Pock:JMIV12}.

Our paper proposes new discrete MRF models for approximating curvature regularization terms like $\int_{\partial S} |\kappa| d\sigma$. 
We primarily focus on the absolute curvature. Unlike length or squared curvature regularization, this term does not add
shrinking or ballooning bias. 

Our technique evaluates curvature using small patches either on a grid or on 
a cell complex, as illustrated in Fig.\ref{fig:gridVScomplex}. In case of a grid, our patches use a novel {\em integral geometry}
approach to evaluating curvature. In case of a complex, our patch-based approach can use standard geometry for evaluating 
curvature. The relationship to previous discrete MRF models for curvature is discussed in Section \ref{sec:related_curvature}.

We also propose a very simple and efficient optimization technique, {\em partial enumeration}, directly applicable 
to curvature regularization and many other complex (e.g. high-order or non-submodular) problems. 
Partial enumeration aggregates the graph nodes within some overlapping patches. While the label space of each patch
is larger compared to individual nodes, the interactions between the patches become simpler.  
Our approach can reduce high-order discrete energy formulations to pair-wise {\em Constraint Satisfaction Problem} 
with unary costs (uCSP). The details of our technique and related work are in Section~\ref{sec:graphconst}.

\begin{figure*}
\begin{center}
\begin{tabular}{ccc}
\includegraphics[height=57mm]{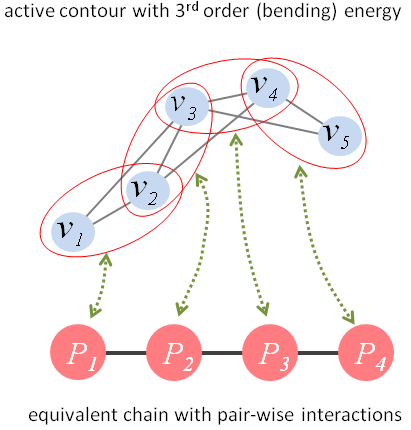} &
\includegraphics[height=57mm]{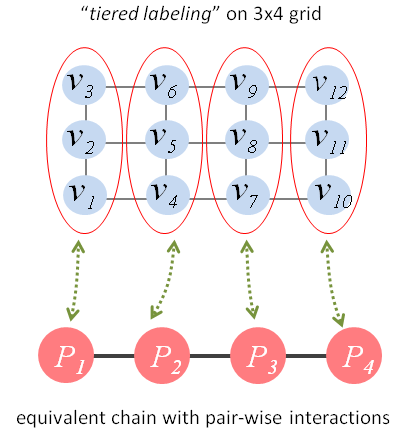} &
\includegraphics[height=57mm]{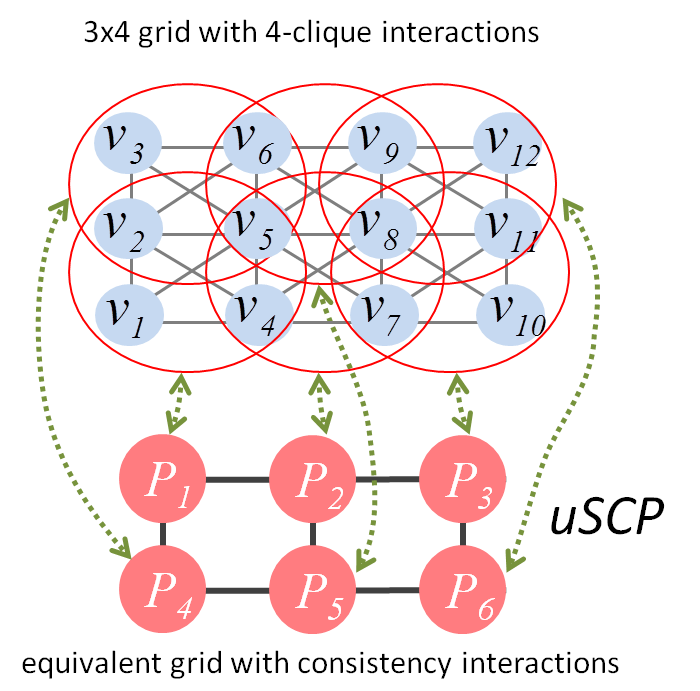} \\
(a) active contours & (b) tiered labeling & (c) more general
\end{tabular}
\end{center}
\caption{Examples of {\em partial enumeration}. Some instances of partial enumeration can be found in 
prior art: 3-rd order active contour model and  {\em tiered labeling} energy on a grid  
\cite{felzenszwalb-veksler-cvpr-2010}  are reduced to pairwise energies on a chain that can be optimized 
by dynamic programming (a-b). {\em Junction tree} approach to energy minimization \cite{Koller:09} can also be seen as a specific form of 
partial enumeration. In general, we show that partial enumeration could be useful for simplifying 
a wide class of complex (non-submodular or high-order) optimization problems without reducing them 
to a chain or a tree. For example (c), if high-order factors on the upper grid are covered by overlapping 2x2 patches, than 
the equivalent optimization problem on the lower graph needs only pairwise interactions
enforcing consistency between the super-nodes representing the patches. }
\label{fig:PEexamples}
\end{figure*}

Some specific examples of {\em partial enumeration} can be found in prior art. For example,
to optimize a {\em snake} with a bending (3rd-order) energy it is customary to combine each pair of adjacent control points
into a single super-node, see Fig.\ref{fig:PEexamples}(a). If the number of states for each control point is $m$ then the number 
of states for the super-node is $m^2$. Thus, the label space has increased. On the other hand, the 3rd-order bending energy on the original snake is simplified to a pair-wise energy on the chain of super-nodes, which can be 
efficiently solved with dynamic programming in $O(nm^3)$. Analogous simplification of a high-order {\em tiered labeling} 
energy on a grid graph to a pair-wise energy on a chain was proposed in \cite{felzenszwalb-veksler-cvpr-2010}, 
see Fig.\ref{fig:PEexamples}(b). Their approach can also be seen as special case of partial enumeration, even though 
non-overlapping ``patches'' are sufficient in their case. 
We study {\em partial enumeration} as a more general technique for simplifying complex 
(non-submodular or high-order) energies without necessarily reducing the problems to chains/trees, 
see Fig.\ref{fig:PEexamples}(c).

Our contributions can be summarized as follows:
\begin{itemize} \vspace{-1ex}
\item simple patch-based models for curvature \vspace{-1.5ex}
\item integral geometry technique for evaluating curvature \vspace{-1.5ex}
\item easy-to-implement partial enumeration technique reducing patch-based MRF models 
to a pairwise {\em Constraint Satisfaction Problem} with unary costs directly addressable with 
many approximation algorithms \vspace{-1.5ex}
\item our uCSP modification of TRWS outperforms several alternatives
producing near-optimal solutions with smaller optimality gap and shorter running times \vspace{-1ex}
\end{itemize}
The experiments in Sections \ref{sec:graphconst} and \ref{sec:exp} show that our patch-based technique obtains 
state-of-the-art results not only for curvature-based segmentation, but also for high-order stereo
and deconvolution problems.
\begin{figure*}[htb]
\begin{center}
\small
\begin{tabular}{c@{\hspace{5ex}}c}
(a) curvature patches on a cell complex (basic geometry) & 
(c) curvature patches on a pixel grid  (integral geometry)  \\[1ex]
\includegraphics[width=50mm]{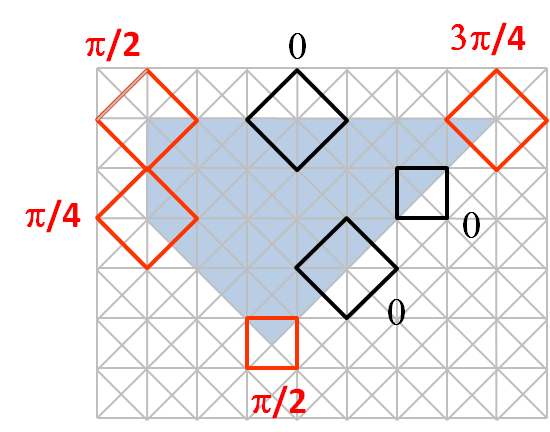}  &
\includegraphics[width=55mm]{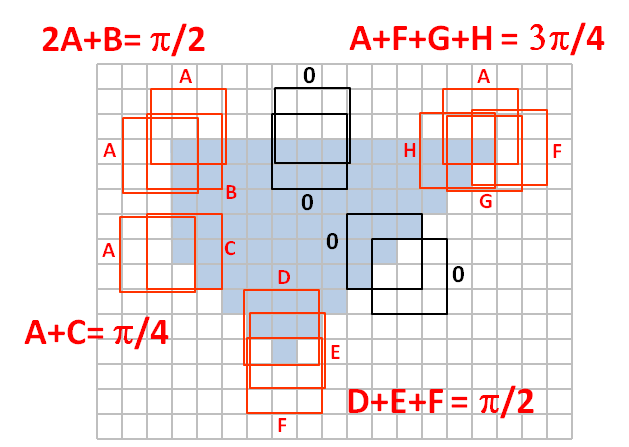} \\[1ex]
(b) our cell-complex patches (8-connected), & (d) our pixel-grid patches (3x3), \\
up to symmetries, and resulting segmentation. & up to symmetries, and resulting segmentation.
\end{tabular}
\begin{tabular}{cc@{\hspace{8ex}}cc}
\includegraphics[height=17mm]{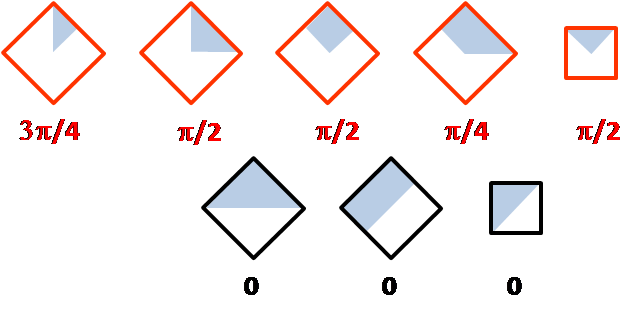}  &
\includegraphics[width=35mm]{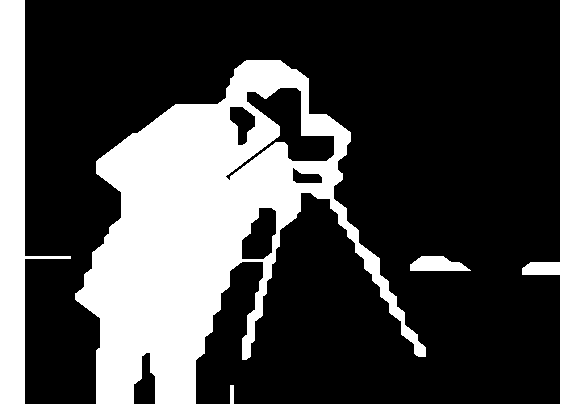} &
\includegraphics[height=17mm]{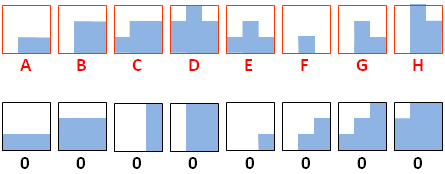} &
\includegraphics[width=35mm]{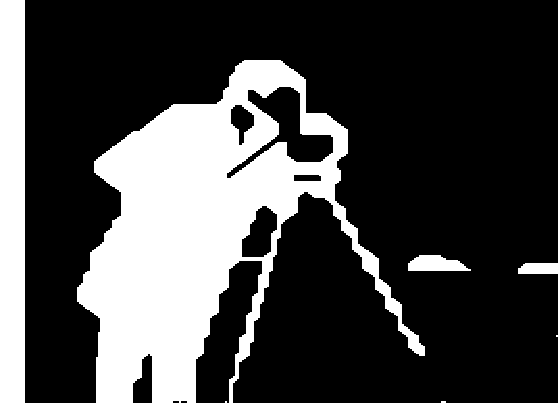} 
\end{tabular}
\end{center} \vspace{-3ex}
\caption{\small Evaluating curvature of a segment on a complex (a,b) and on a grid (c,d) 
using standard and integral geometry.
At sufficiently high resolution, any segment $C$ is a polygon. Thus, to minimize
curvature functionals like $\int_C |\kappa| ds$ we need to evaluate all corners.
We use (overlapping) patches created for each vertex on a complex (a) and for each pixel on a grid (c).
A patch on a complex (a,b) consists of all cells adjacent to a vertex and a grid patch (c,d) is a square window 
centered at a pixel. For $\pi/4$ precision as on 8-complex (a), we use 3x3 windows on a grid (b). 
For finer $\pi/8$ precision as on 16-complexes, we use 5x5 windows. 
Note that each corner on a complex (a) can be directly evaluated from a configuration (labeling) 
of a single patch using standard geometry. However,  each corner on a grid (c) should be evaluated using 
integral geometry by summing over multiple patches covering the corner. 
Patch configurations in black occur at straight boundaries and should contribute zero weights. 
Patch configurations in red correspond to curved boundaries. The weights $A,\dots,H$ for all such configurations (d) 
can be pre-computed from a system of linear equations for all types of corners. 
The accuracy of integral geometry approach to curvature on a grid is comparable 
to the standard basic geometry used on complexes, see (b) and (d).}
\label{fig:gridVScomplex}
\vspace{-3mm}
\end{figure*}

\section{Curvature on patches and related work} \label{sec:related_curvature}

We discuss approximation of curvature in the context of binary segmentation 
with regularization energy
\begin{equation} \label{curvatureint}
E(S) = \int_{\text{int}(S)} f(x) \ dx + \int_{\partial S} \lambda|\kappa| d\sigma,
\end{equation}
where $\kappa$ is curvature, $\lambda$ is a weighting parameter, and unary potential $f(x)$ is a data term. 

Our grid-patches in Fig.\ref{fig:gridVScomplex}(c) and our complex-patches in Fig.\ref{fig:gridVScomplex}(a)
can be seen as ``dual'' methods for estimating curvature in exactly the same way as {\em geo-cuts}  \cite{BK:iccv03} 
and complex-based approach in \cite{Sullivan:thesis92} are ``dual'' methods for evaluating geometric length.
Our grid-patch approach to curvature extends ideas in {\em geo-cuts}  \cite{BK:iccv03} that showed how discrete MRF-based regularization methods can use {\em integral geometry} to accurately approximate length via Cauchy-Crofton formula.
We show how general integral geometry principles can also be used to evaluate curvature, see Fig.\ref{fig:gridVScomplex}(c).
The complex-patch technique in Fig.\ref{fig:gridVScomplex}(a) uses an alternative method for approximating 
curvature based on standard geometry as in  \cite{schoenemann-etal-ijcv-2012,elzehiry-grady-cvpr-2010,strandmark-kahl-emmcvpr-2011}.

Our patch-based curvature models could be seen as extensions of {\em functional lifting} 
\cite{Pock:JMIV12} or {\em label elevation} \cite{olsson:CVPR12}. Analogously to the {\em line processes} in 
\cite{geman-geman-pami-1984}, these second-order regularization methods use variables 
describing both location and orientation of the boundary. Thus, their curvature is the first-order (pair-wise) energy.
Our patch variables include enough information about the local boundary to reduce the curvature to unary terms.

Curvature is also reduced to unary terms in \cite{schoenemann-etal-ijcv-2012} using auxiliary variables for
each pair of adjacent {\em line processes}. Their integer LP approach to curvature is formulated over a large number of
binary variables defined on fine geometric primitives (vertexes, faces, edges, pairs of edges, etc), which are tied
by constraints. In contrast, our unary representation of curvature uses larger scale geometric primitives 
(overlapping patches) tied by consistency constraints. The number of corresponding variables is significantly smaller, but
they have a much larger label space. Unlike \cite{schoenemann-etal-ijcv-2012} and us,
\cite{elzehiry-grady-cvpr-2010,strandmark-kahl-emmcvpr-2011} represent curvature via high-order interactions/factors. 

Despite technical differences in the underlying formulations and optimization algorithms, our patch-based approach 
for complexes in Fig.\ref{fig:gridVScomplex}(a) and \cite{schoenemann-etal-ijcv-2012,strandmark-kahl-emmcvpr-2011} 
use geometrically equivalent models for approximating curvature. That is, all of these models would produce the same solution, 
if there were exact global optimization algorithms for them. 
The optimization algorithms for these models do however vary, both in quality, memory, and run-time efficiency.

In practice, grid-patches are easier to implement than complex-patches
because the grid's regularity and symmetry.
While integral geometry estimates curvature on a pixel grid as accurately 
as the standard geometry on a cell complex, see Figs.\ref{fig:gridVScomplex}(b,d), in practice, our proposed optimization 
algorithm for the corresponding uCSP problems works better (with near-zero optimality gap) for the grid version of our method. 
To keep the paper focused, the rest of the paper primarily concentrates on grid-based patches. 

Grid patches were also recently used for curvature evaluation in \cite{shekhovtsov-etal-dagm-2012}.
Unlike our integral geometry in Fig.\ref{fig:gridVScomplex}(c), their method computes a minimum response over 
a number of affine filters encoding some learned ``soft'' patterns. The response to each filter combines 
deviation from the pattern and the cost of the pattern. The mathematical justification of this approach to curvature estimation
is not fully explained and several presented plots indicate its limited accuracy. As stated in \cite{shekhovtsov-etal-dagm-2012},
``the plots do also reveal the fact that we consistently overestimate the true curvature cost.''
The extreme ``hard'' case of this method may reduce to our technique if the cost of each pattern
is assigned according to our integral geometry equations in Fig.\ref{fig:gridVScomplex}(c). However, this case
makes redundant the filter response minimization and the pattern costs learning, which are the key 
technical ideas in \cite{shekhovtsov-etal-dagm-2012}.

\section{Simple Patch-based Optimization} \label{sec:graphconst}

One way to optimize our patch-based curvature model is to formulate the optimization problem on 
the original image pixel grid $\langle V,\C\rangle$ in Figure~\ref{fig:PEexamples}(c, top grid) using
pixel variables $\x=\{x_i|i\in V\}$, high-order factor $\alpha\in\C$, and energy
\begin{equation}
E(\x) = \sum_{\alpha\in\C} E_\alpha (\x_\alpha)
\label{eq:clusterenergy}
\end{equation}
where $\x_\alpha=\{x_i|i\in\alpha\}$ is the restriction of $\x$ to $\alpha$.
Optimization of such high-order energies is generally NP-hard, but a number of existing approximate algorithms 
for certain high-order MRF energies could be applied. Our experimental section includes the results of some generic 
methods  \cite{GTRWS:arXiv12,kahl-strandmark-dam-2012}  that have publicly available code. 

We propose a different approach for optimizing our high-order curvature models that equivalently reformulates 
the problem on a new graph, see Figure~\ref{fig:PEexamples}(c, bottom grid). The motivation is as follows. 
One naive approach applicable to NP-hard high-order energies on small images 
is exhaustive search that enumerates all possible labelings of the underlying pixel graph.
On large problems one can use partial enumeration to simplify high-order problems. 
If some set of relatively small overlapping patches covers all high-order factors,
we can build a new graph where nodes correspond to patches and their labels enumerate patch states, 
as in Figure~\ref{fig:PEexamples}(c, bottom grid). Note that high-order interactions reduce to unary potentials, 
but, due to patch overlap, hard pair-wise consistency constraints must be enforced. 

Our general approach transforms a high-order optimization problem to a pair-wise Constraint Satisfaction Problem 
with unary costs (uCSP).
Formally, the corresponding energy could be defined on graph $\langle \V,\E\rangle$ in Figure~\ref{fig:PEexamples}(c, bottom grid) 
where nodes correspond to a set of patches $\V$ with the following property: for every factor $\alpha\in\C$ 
there exists patch $\beta\in\V$ such that $\alpha\subseteq\beta$. For example, $\V=\C$ works, but, in general, 
patches in  $\V$ can be bigger than factors in $\C$. We refer to nodes in $\V$ as super nodes. 
Clearly, \eqref{eq:clusterenergy} could be equivalently rewritten as an energy with unary and pairwise terms:
\begin{equation}
E_\text{super}(\X) = \sum_{\alpha \in \V} U_\alpha (X_\alpha) + \sum_{(\alpha,\beta) \in \E} P_{\alpha \beta}(X_\alpha,X_\beta)
\label{eq:superenergy}
\end{equation}
The label $X_\alpha$ of a super node $\alpha$ corresponds to the state of all individual pixels $\x_\alpha$ within the patch.
By enumerating all possible pixel states within the patch we can now encode the higher order factor 
$E_\alpha(x_\alpha)$ into the unary term $U_\alpha(X_\alpha)$ of \eqref{eq:superenergy}.
The pairwise consistency potential $P_{\alpha\beta}(X_\alpha,X_\beta)=0$ if variables 
$X_\alpha$ and $X_\beta$ agree on the overlap $\alpha\cap\beta$, and $P_{\alpha\beta}(X_\alpha,X_\beta)=+\infty$ otherwise. 
The set of edges $\E$ may contain all pairs 
$\{\alpha,\beta\}\subset\V$ such that $\alpha\cap\beta\ne\varnothing$, but a smaller could be enough. For example,
the graph in Figure~\ref{fig:PEexamples}(c, bottom grid) 
does not need diagonal edges. A formal procedure for selecting the set of edges
is given in Appendix~\ref{app:lprelaxation}.

Optimization of pairwise energy \eqref{eq:superenergy} can be addressed with standard methods 
like~\cite{kolmogorov-pami-2006,Globerson:NIPS07} that can be modified for our specific consistency constraints 
to gain significant speed-up (see Sec.\ref{sec:speedup}).

\noindent{\bf LP relaxations~~}
When we apply method like TRW-S~\cite{kolmogorov-pami-2006} to energy~\eqref{eq:superenergy}, we essentially solve a higher-order relaxation of the original energy~\eqref{eq:clusterenergy}.
Many methods have been proposed in the literature for solving higher-order relaxations, e.g.~\cite{sontag-etal-uai-2008,komodakis-etal-cvpr-2009,Meltzer:UAI09,Werner:PAMI10,GTRWS:arXiv12} to name just a few.
To understand the relation to these methods, in  Appendix~\ref{app:lprelaxation} we analyze which specific relaxation is solved by our approach.
We then argue that the complexity of message passing in our scheme roughly matches that
of other techniques that solve a similar relaxation. 
\footnote{Message
 passing techniques require the minimization of expressions of the form
 $E_\alpha(\x_\alpha)+\ldots$ where dots denote lower-order factors.
 Here we assume that this expression is minimized by going through all possible labellings $\x_\alpha$.
 This would hold if, for example, $E_\alpha(\cdot)$ is represented by a table (which is the case
 with curvature). Some terms $E_\alpha(\cdot)$ used in practice have a special structure
 that allow more efficient computations; in this case other techniques may have a better
 complexity. One example is {\em cardinality-based potentials}~\cite{Tarlow:AISTATS10} which can have a very high-order.}
In practice, the choice of the optimization method is often motivated by the ease of implementation;
we believe that our scheme has an advantage in this respect, and thus may be preferred by practitioners.

\noindent{\bf Other related work~~}
The closest analogue of our approach is perhaps the ``hidden transformation'' approach~\cite{Bacchus:AI02} that converts
an arbitrary CSP into a pairwise CSP (also known as the ``constraint satisfaction dual problem'').
We provide a {\em weighted} version of this transformation; to our knowledge, this has not been studied yet, 
and the resulting relaxation has not been analyzed.

Our method bears some resemblance to the work~\cite{komodakis-etal-cvpr-2009} that also uses square patches.
However, we believe that the relaxation solved in~\cite{komodakis-etal-cvpr-2009} is weaker than ours; details
are discussed in the Appendix~\ref{app:lprelaxation}.

Researchers also considered alternative techniques for converting a high-order energy of binary variables into a pairwise one.
We will compare to one such technique, \cite{kahl-strandmark-dam-2012}, which generalizes roof duality to factors of order 3 and 4.

\subsection{Application to $\pi / 2$-precision curvature}
\begin{figure*}[htb]
\small
\vspace{-2mm}
\begin{center}
\includegraphics[width=25mm]{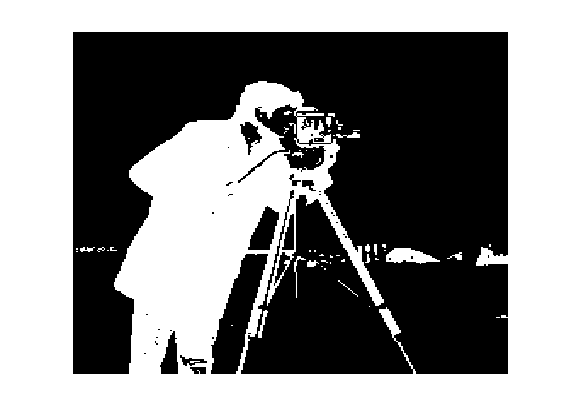}
\includegraphics[width=25mm]{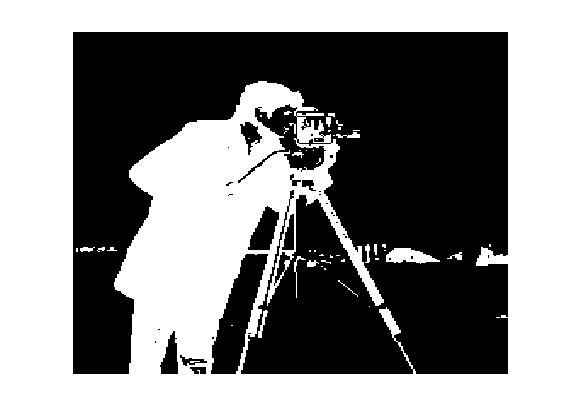}
\includegraphics[width=25mm]{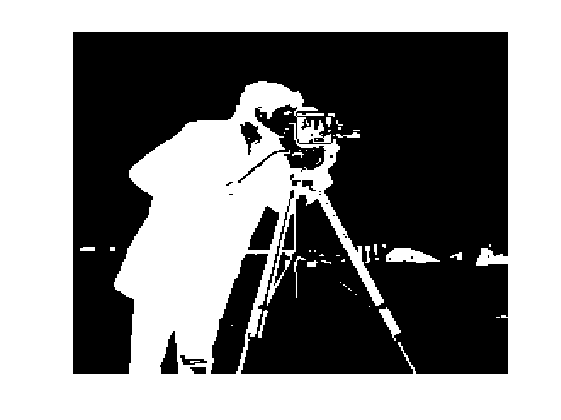}
\includegraphics[width=25mm]{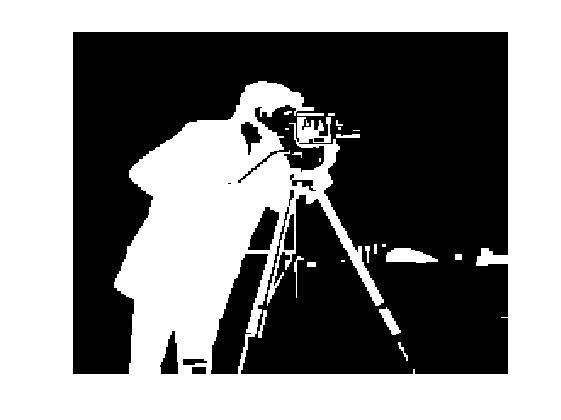}
\includegraphics[width=25mm]{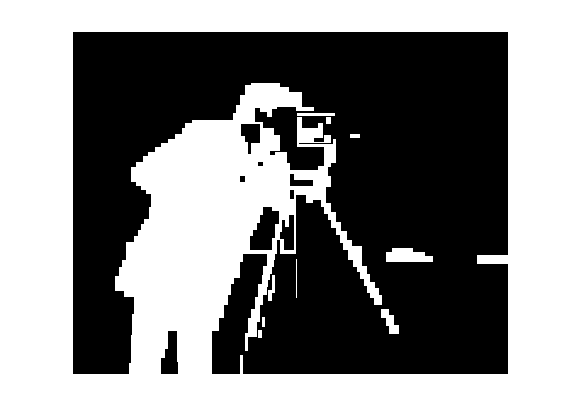}
\includegraphics[width=25mm]{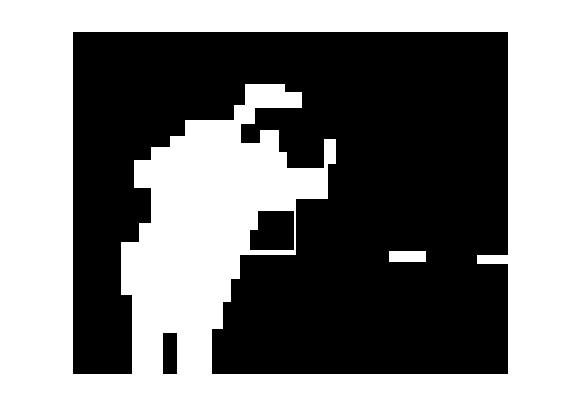}
\setlength{\tabcolsep}{3pt}
\begin{tabular}{lccccc}
\toprule
$\lambda$ &TRW-S Energy & TRW-S Lower bound &  Unlabled by GRD(-heur)  &  TRW-S running time & GRD(-heur) running time\\
\midrule
$2.5\cdot10^{-5}$ &    $4677$  &  $4677$  &  $0\%$ ($0\%$)  & $0.934$s &  $10737$s ($7.08$s)\\
$2.5\cdot10^{-4}$ &   $4680$ &  $4680$ &   $0\%$ ($ 0\%$)  & $0.961$s & $9287$s ($10.7$s)\\
$2.5\cdot10^{-3}$ &   $4707$ &  $4707$ &   $0.2\%$ ($0.2\%$) & $2.43$s & $10731$s ($7.32$s)\\
$2.5\cdot10^{-2}$ &    $4910$  &  $4910$  & $6.7\%$ ($6.8\%$) & $3.98$s & $12552$s ($6.96$s)\\
$2.5\cdot10^{-1}$  &  $5833$   & $5833$  & $100\%$ ($100\%$) & $14.3$s & $12337$s ($10.9s$)\\
$2.5\cdot10^{0}$ &   $7605$ &   $ 7605$ & $100\%$ ($100\%$) & $28.8$s & $7027$s ($22.2$s)\\
\bottomrule
\end{tabular}
\end{center}
\vspace{-5mm}
\caption{\small Our results for $2\times2$ curvature with different regularization weight $\lambda$ (top row of images). 
TRW-S with super nodes gives a tight lower bound. The figures within parenthesis are results when using the heuristics proposed for speedup in \cite{kahl-strandmark-dam-2012}). }
\label{GRDcomparison}
\end{figure*}
\setlength{\tabcolsep}{6pt}

In this section we illustrate our approach on a very coarse approximation of curvature where we only allow boundary edges that are either 
horizontal or vertical. It is shown in \cite{elzehiry-grady-cvpr-2010} that the resulting energy can be formulated as in \eqref{eq:clusterenergy} where $\C$ contains the edges of an 8-connected neighborhood, see Fig.~\ref{fig:orggraph}.
In contrast we formulate the problem as \eqref{eq:superenergy} where $\V$ is the set of all $2\times2$ patches.
Consider the patches in Figure~\ref{basepatches} and their curvature estimates.
\begin{figure}[htb]
\begin{center}
\begin{tabular}{cccccc}
\includegraphics[width=8mm]{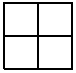} &
\includegraphics[width=8mm]{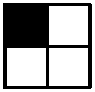} &
\includegraphics[width=8mm]{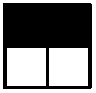} &
\includegraphics[width=8mm]{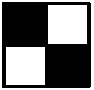} 
\end{tabular}
\end{center}
\vspace{-6mm}
\caption{\small Four of the 16 patch states used to encode curvature with penalties $0$, $\pi/2$, $0$ and $2\pi$ respectively.}
\label{basepatches}
\end{figure}
The patches have 4 pixel boundaries that intersect in the middle of the patch. To compute the curvature contribution of a patch we need to determine which of the 4 pixel boundaries also belong to the segmentation boundary.
If two neighboring pixels (sharing a boundary) have different assignments then their common boundary belongs to the segmentation boundary.

Figure \ref{fig:newgraph} shows the approach.
\begin{figure}
\begin{center}
\begin{tabular}{cc}
\includegraphics[width=25mm]{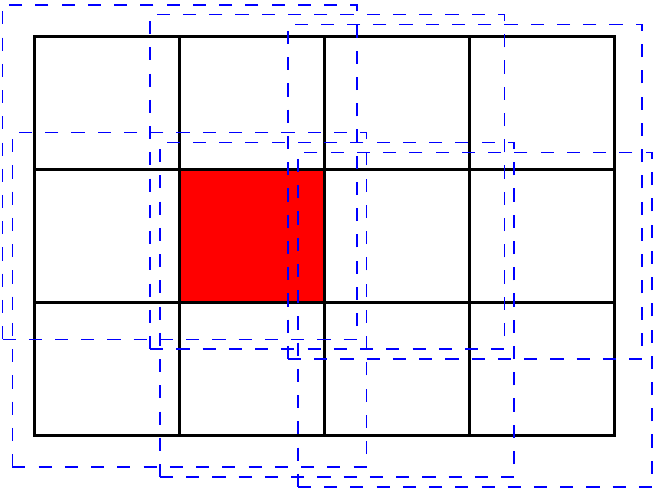} &
\includegraphics[width=35mm]{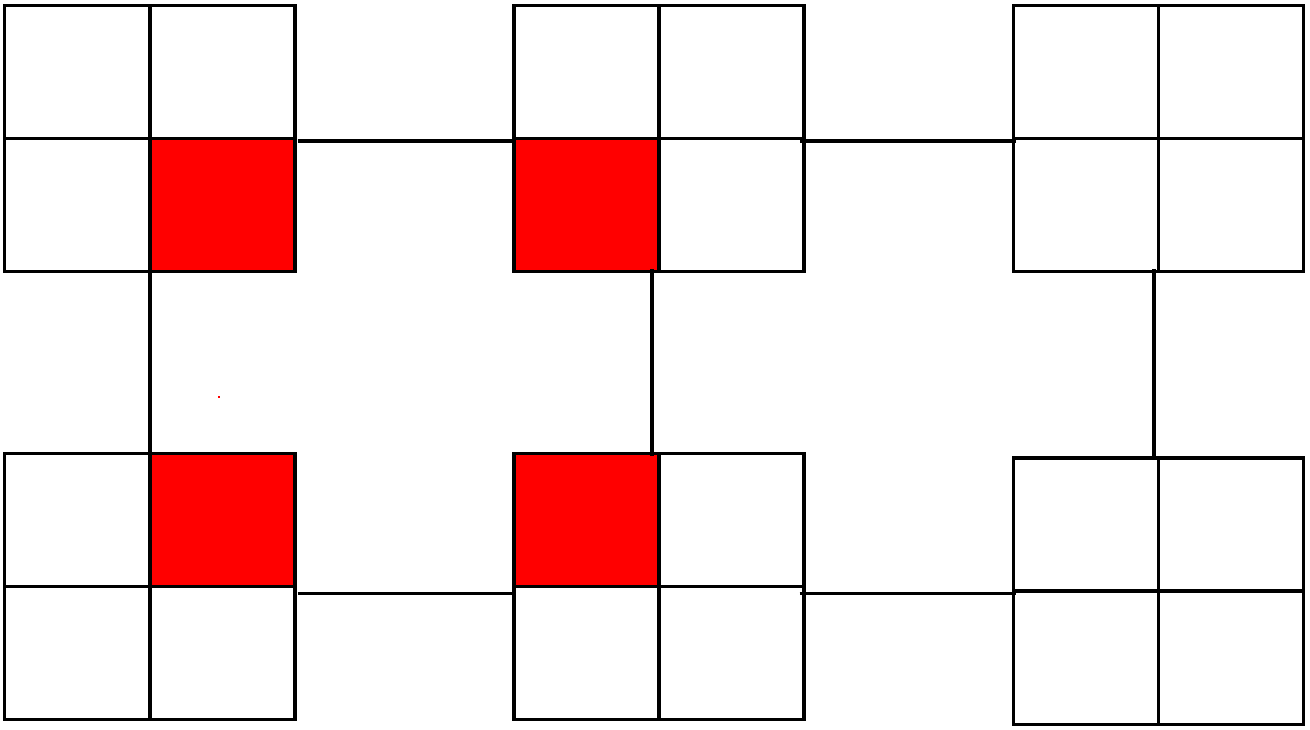} 
\end{tabular}
\end{center}
\vspace{-5mm}
\caption{\small Super nodes formed in a sliding window fashion. The red pixel occurs in 4 super nodes. Pairwise interactions ensure that shared pixels are assigned the same value.}
\label{fig:newgraph}
\end{figure}
We start by forming patches of size $2\times2$ into super nodes. This is done in a sliding window fashion, that is, super node $(r,c)$ consists of the nodes $(r,c)$, $(r,c+1)$, $(r+1,c)$ and $(r+1,c+1)$, where $r$ and $c$ are the row and column coordinates of the pixels. 

Each super node label can take 16 values corresponding to states of the individual pixels.
The curvature interaction and data terms of the original problem are now transformed to unary potentials. Note that since patches are overlapping pixels can be contained in up to four super nodes. 
In order not to change the energy we therefore weight the contribution from the original unary term, $f(x)$ in \eqref{curvatureint}, to each patch such that the total contribution is 1.
For simplicity we give pixels that occur $k$ times the weight $1/k$ in each super node.

Finally to ensure that each pixel takes the same value in all the super nodes where it is contained we add the "consistency" edges $\E$ between neighboring super nodes (see Fig. \ref{fig:newgraph}). 
Note it is enough to use a 4-connected neighborhood.

Our two approaches from Figure~\ref{fig:gridVScomplex} and \cite{schoenemann-etal-ijcv-2012,elzehiry-grady-cvpr-2010,strandmark-kahl-emmcvpr-2011} all assign the same curvature costs for the patches in Figure~\ref{basepatches}.
Therefore, assuming that the global minimum can be found, they yield the same solution for $\pi/2$-precision curvature.

\subsection{Efficient Message Passing} \label{sec:speedup}
Since the number of labels can be very large when we have higher order factors it is essential to compute messages fast.
The messages sent during optimization has the form
\begin{equation}
m_{\alpha\beta}^t (X_\beta) = \min_{X_\alpha} (P_{\alpha\beta}(X_\alpha,X_\beta) + h(X_\alpha)),
\end{equation}
where $h$ is some function of super node label $X_\alpha$.

To compute the message we order the labels of both node $\alpha$ and $\beta$ into (disjoint) groups according to the assignments of the shared pixels.
The message values $m_{\alpha\beta}^t(X_\beta)$ for all the $X_\beta$ in the same group can now be found by searching for the smallest value of $h(X_\alpha)$ in the group consistent with the $X_\beta$.
The label order depends on the direction of the edge between $\alpha$ and $\beta$, however it does not change during optimization and can therefore be precomputed at startup.
The bottleneck is therefore searching the groups for the minimal value which can be done in linear time.

Note that this process does not require that all the possible patch assignments are allowed.
For larger patches (see Section~\ref{superduper}) some of the patch states may not be of interest to us and the corresponding labels can simply be removed.

Figure \ref{GRDcomparison} compares our approach for $\pi/2$-precision curvature to Generalized Roof Duality (GRD) \cite{kahl-strandmark-dam-2012}.
We used TRW-S \cite{kolmogorov-pami-2006} with $2\times 2$ patches and our efficient message passing scheme.
Our approach gives no duality gap.

\subsection{Lower Bounds using Trees}

As observed in \cite{elzehiry-grady-cvpr-2010} the $2\times2$ curvature interaction reduces to pairwise interactions between all the pixels in the patch.
In this discrete setting \eqref{curvatureint} reduces to \eqref{eq:clusterenergy}
where $\C$ consists of the edges of the underlying (8-connected) graph, see Figure \ref{fig:orggraph}. 
Therefore it could in principle be solved using roof duality (RD) \cite{rother-etal-cvpr-2007} or TRW-S \cite{kolmogorov-pami-2006}. (Note that this is only true for this particular neighborhood and the corresponding interaction penalty.)
However, it may still be useful to form super nodes. 
Methods such as \cite{kolmogorov-pami-2006} work by decomposing the problem into subproblems on trees and combining the results into a lower bound on the optimal solution. 
Sub-trees with super nodes are in general stronger than regular trees.

\begin{figure}[htb]
\begin{center}
\begin{tabular}{cc}
\includegraphics[width=20mm]{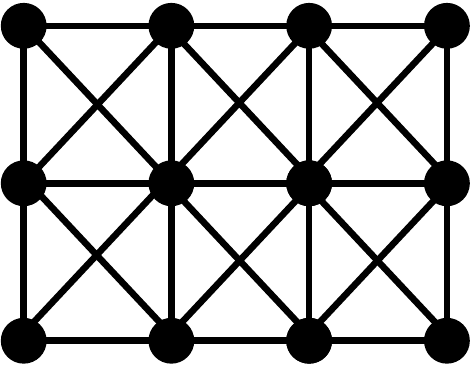} &
\includegraphics[width=20mm]{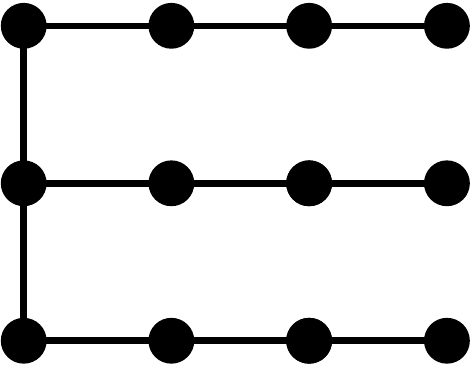}
\end{tabular}
\end{center}
\vspace{-5mm}
\caption{\small {\em Left}: 8-connected graph. {\em Right}: Sub-tree $\tree$.}
\label{fig:orggraph}
\end{figure}
\vspace{-5mm}
\begin{figure}[htb]
\begin{center}
\includegraphics[width=60mm]{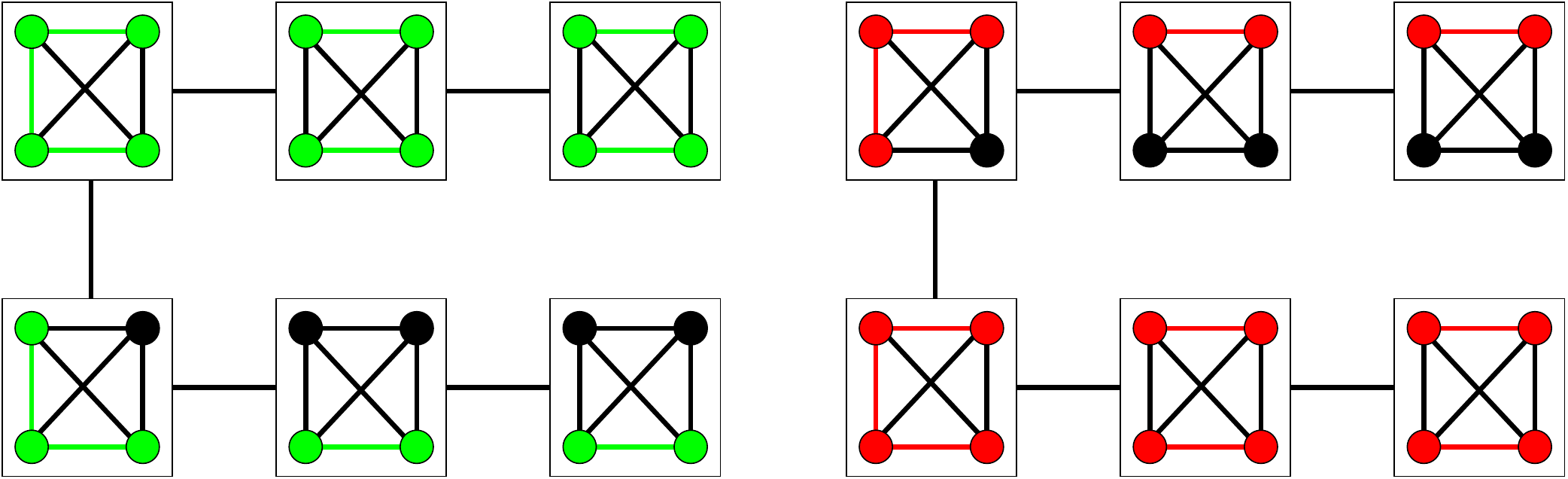} 
\end{center}\vspace{-3mm}
\caption{\small
$\tree_{2\times2}$ contains two copies (red and green) of $\tree$ .}
\label{fig:newtree}
\end{figure}

\begin{figure*}[htb]
\begin{center}
\begin{tabular}{cccc}
(a) & (b) & (c) & (d) \\
\includegraphics[width=0.2\linewidth,height=0.2\linewidth]{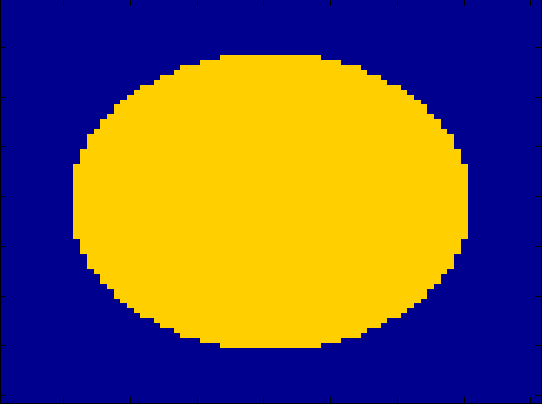} &
\includegraphics[width=0.2\linewidth,height=0.2\linewidth]{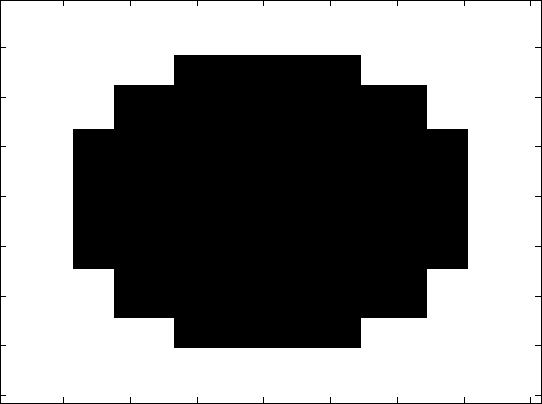} &
\includegraphics[width=0.2\linewidth,height=0.2\linewidth]{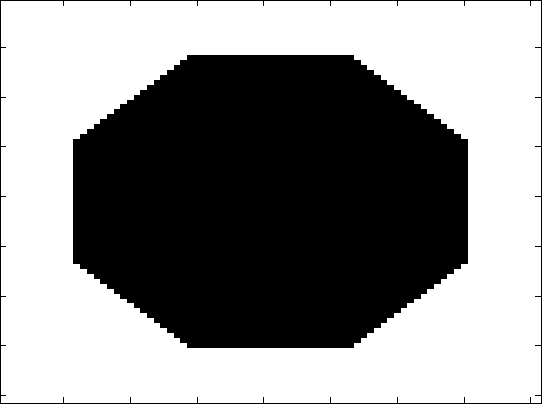} &
\includegraphics[width=0.2\linewidth,height=0.2\linewidth]{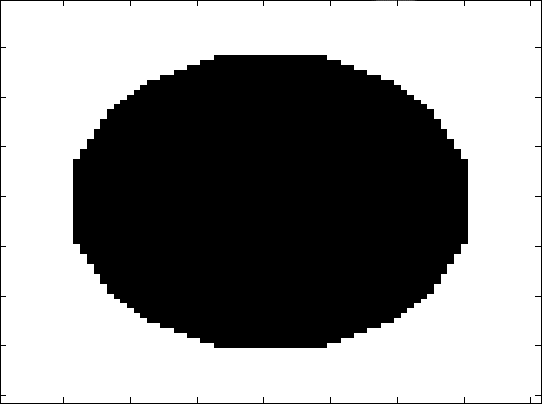} 
\end{tabular}
\end{center}\vspace{-5mm}
\caption{\small Segmentation results on a $81\times81$ pixel image using different patch sizes with same very high regularization weight ($\lambda = 20$). (a) - Data term, (b) - $2 \times 2$ patches clearly favors horizontal and vertical boundaries , (c) - $3 \times 3$ patches, favors directions that are multiples of $\pi/4$, (d) $5 \times 5$ patches, favors directions that are multiples of $\pi/8$.}
\label{circle}
\end{figure*}

Consider for example the sub-tree $\tree$ in Figure \eqref{fig:orggraph}. We can form a similar sub-tree $\tree_{2\times2}$ using the super nodes, see Figure \ref{fig:newtree}.
Note that the edges that occur twice within the super nodes have half the weight of the corresponding edges in Figure~\ref{fig:orggraph}.
Looking in the super nodes and considering the consistency edges we see that we can find two instances of $\tree$ within $\tree_{2 \times 2}$
(see Figure~\ref{fig:newtree}) both with weights 1/2 (here the edges that have weight 1 are allowed to belong to both trees). 
Hence if we view these two instances as independent and optimize them we get the same energy as optimization over $\tree$ would give.
In addition there are other edges present in $\tree_{2\times2}$, and therefore this tree gives a stronger bound.

In a similar way, we can construct even stronger trees by increasing the patch size further (event though the interactions might already be contained in the patches).
If we group super nodes in a sliding window approach we obtain a graph with $3\times3$ patches, see Figure~\ref{fig:graph33}. (We refer to the new nodes as super-duper nodes.)
If we keep repeating this process we will eventually end up enumerating the entire graph, so it is clear that the resulting lower bound will eventually approach the optimum.
\begin{figure}[htb]
\begin{center}
\includegraphics[width=40mm]{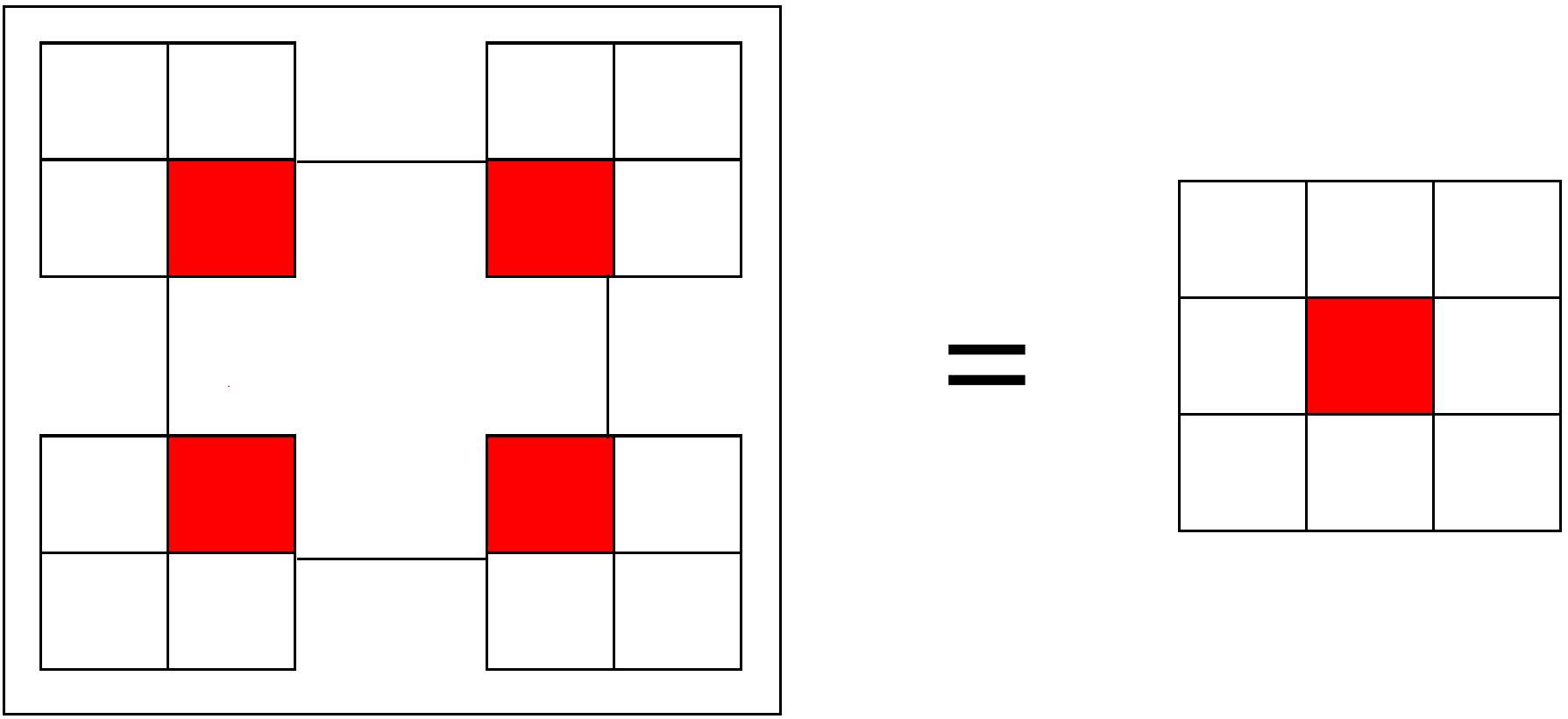}
\end{center}\vspace{-5mm}
\caption{\small Super-duper nodes containing patches of size $3 \times 3$ are created by grouping super nodes of size $2 \times 2$ in groups of $2 \times 2$ in a sliding window fashion.}
\label{fig:graph33}
\end{figure}

In Table~\ref{grady-approach} the same problem as in Figure~\ref{GRDcomparison} is solved using TRW-S without forming any super nodes.
\setlength{\tabcolsep}{4pt} 
\begin{table}[htb]
\small
\begin{center}
\begin{tabular}{lrr}
\toprule
$\lambda$ & Energy & Lower bound\\
\midrule
$2.5\cdot10^{-5}$ & 4677 & 4677\\
$2.5\cdot10^{-4}$ & 4680 & 4680\\
$2.5\cdot10^{-3}$ & 4709 & 4705\\
$2.5\cdot10^{-2}$ & 5441 & 4501\\
$2.5\cdot10^{-1}$ & 16090& -16039\\
$2.5\cdot10^{0}$ & 15940 & -19990\\ 
\bottomrule
\end{tabular}
\end{center}\vspace{-5mm}
\caption{\small Same as in Figure~\ref{GRDcomparison} without super nodes.}
\label{grady-approach}
\end{table}
\setlength{\tabcolsep}{6pt} 

\subsection{Application to $\pi/4$ and $\pi/8$ precision curvature}\label{superduper}

For patches of size $2 \times 2$ it is only possible to encourage horizontal and vertical boundaries. 
Indeed, along a diagonal boundary edge all $2\times2$ patches look like the second patch in Figure~\ref{basepatches}.
To make the model more accurate and include directions that are multiples of $\pi/4$ radians we will look at patches of a larger size, see Figure \ref{fig:gridVScomplex}(c).

For multiples of $\pi/4$ radians it is enough to have $3 \times 3$ patches and for $\pi/8$ radians we use patches of size $5 \times 5$.
However, the number of distinct patch-labels needed to encode the interactions (transitions between directions) is quite high. 
It is not feasible to determine their costs by hand.

To compute the specific label costs we generate representative windows of size slightly larger than the patch 
(for $3\times3$ patches we use $5 \times 5$ windows)  that contain either a straight line or a transition between two directions 
of known angle difference. 
From this window we can determine which super node assignments occur in the vicinity of different transitions.
We extract all the assignments and constrain their sum, as shown in Figure \ref{fig:gridVScomplex}, to be the known curvature of the window.
Furthermore, we require that the cost of each label is positive. 
If a certain label is not present in any of the windows we do not allow this assignment.
This gives us a set of linear equalities and inequalities for which we can find a solution (using linear programming).
The procedure gives 122 and 2422 labels for the $3 \times 3$ and $5 \times 5$ cases respectively.
More details are given in  Appendix~\ref{app:patchcost}.

Figures~\ref{circle} illustrates the properties of the different patch sizes.
Here we took an image of a circle and segmented it using the 3 types of approximating patches.
Note that there is no noise in the image, so simply truncating the data term would give the correct result. 
We segment this image using a very high regularization weight ($\lambda = 20$). 
In (b) horizontal and vertical boundaries are favored since these have zero regularization cost.
In (b) and (c) the number of directions with zero cost is increased and therefore the approximation improved with the patch size.
Figure~\ref{cameraman} shows real segmentations with the different patch sizes 
(with $\lambda = 1$).
Table~\ref{camtable} shows energies, lower bounds and execution times for a couple of methods.
Note that all methods except GTRW-S use our super node construction, here we use the single separator implementation \cite{GTRWS:arXiv12}. 
Thus, GTRW-S solves a weaker relaxation of the energy (this is confirmed by the numbers in Table~\ref{camtable}).
GTRW-S requires specifying all label combinations for each factor. For the patch assignments that we do not use to model curvature we specify a large cost (1000) to ensure that these are not selected. 
Furthermore, TRW-S and Loopy belief propagation (LBP) both use our linear time message computation. For comparison TRW-S (g) uses general message computation.
All algorithms have an upper bound of 10,000 iterations. In addition, for TRW-S and GTRW-S the algorithm converges if the lower bound stops increasing. 
For MPLP \cite{sontag-etal-uai-2008} we allow 10,000 iterations of clustering and we stop running if the duality gap is less than $10^{-4}$. 
Figure~\ref{timecurves} shows convergence plots for the $2\times2$ case.
\begin{figure*}[htb]
\begin{center}
\includegraphics[width=0.2\linewidth]{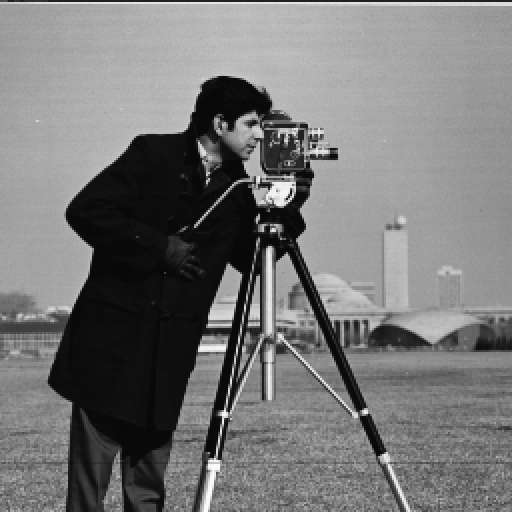}
\includegraphics[width=0.2\linewidth]{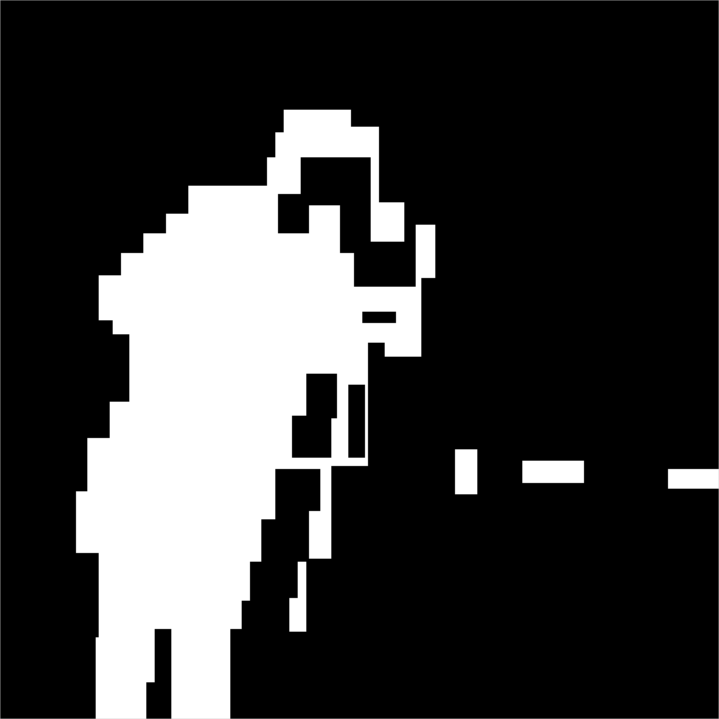} 
\includegraphics[width=0.2\linewidth]{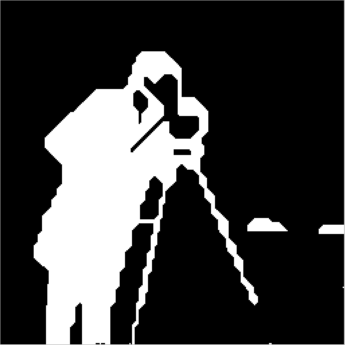} 
\includegraphics[width=0.2\linewidth]{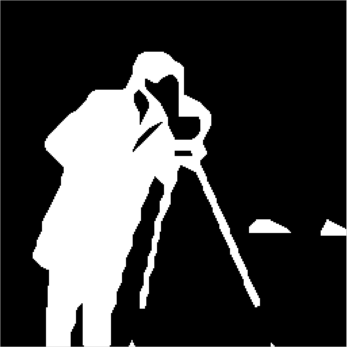} 
\end{center}\vspace{-4mm}
\caption{\small Segmentation of the camera man using (from left to right) $2 \times 2$, $3 \times 3$ and $5 \times 5$ patches with $\lambda = 1$.}
\label{cameraman}
\vspace{-4mm}
\end{figure*}

\begin{table*}[htb]
\small
\subfloat[$2\times 2$ patches.]{
	\begin{tabular}{lrrr}
	\toprule
	& Energy & Lower bound & Time (s) \\
	\midrule
	TRW-S &1749.4 & 1749.4 & 21\\	
    TRW-S (g) & 1749.4 & 1749.4 & 1580 \\	
	MPLP & 1749.4 & 1749.4 & 6584 \\
	LBP & 2397.7 &   & 1565 \\
	GTRW-S & 1867.9  & 1723.8  & 2189 \\
	\bottomrule
	\end{tabular}
} %
\hfill%
\subfloat[$3\times 3$ patches.]{
	\begin{tabular}{rrrr}
	\toprule
	Energy & Lower bound & Time (s) \\
	\midrule
	1505.7 & 1505.7 & 355 \\
	1505.7 & 1505.7 & 41503 \\

	$\ddagger$ & $\ddagger$ & $\ddagger$  \\
		$*$ &  & 3148 \\	
		99840 & 1312.6 & 10785 \\
	\bottomrule
	\end{tabular}
} %
\subfloat[$5\times 5$ patches.]{
	\begin{tabular}{rrrr}
	\toprule
	Energy & Lower bound & Time (s) \\
	\midrule
	1417.0 & 1416.6 & 8829  \\	
	$\ddagger$ & $\ddagger$ &  $\ddagger$ \\
	$\ddagger$ & $\ddagger$  &  $\ddagger$ \\
	$*$ &  & $157532$    \\
	$\ddagger$  &  $\ddagger$  & $\ddagger$  \\
	\bottomrule
	\end{tabular}
}
\caption{\small Cameraman ($256 \times 256$ pixels) with $\lambda = 1$ run with with different path sizes. Resulting segmentation can be seen in Figure~\ref{cameraman}.
$(\ddagger)$ Creating the problem instance not feasible due to excessive memory usage. $(*)$ Inconsistent labeling.}
\label{camtable}
\end{table*}

\begin{figure}[htb]
\graphicspath{{.}}
{\scriptsize
\def\svgwidth{1\linewidth}
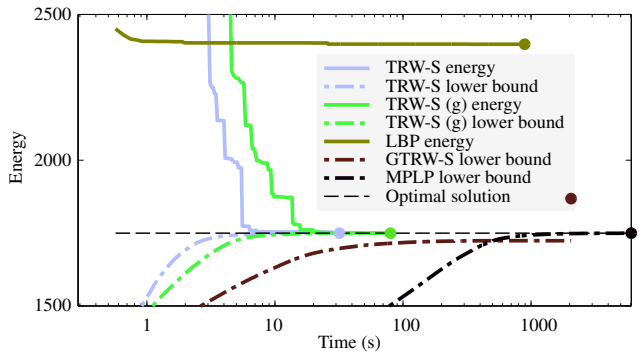
}
\vspace{-5mm}
\caption{\small Logarithmic time plot for energy and lower bound over time for the $2\times2$ experiment in Table~\ref{camtable}. For MPLP and GTRW-S we only show the final energy as a dot.}
\label{timecurves}
\end{figure}

\section{Other Applications}\label{sec:exp}
Our framework does not only work for curvature but applies to a general class of problems. 
In this section we will test our partial enumeration approach for other problems than curvature regularization.

\subsection{Binary Deconvolution}
Figure \ref{deconvolution} (a) shows an image convolved with a $3 \times 3$ mean value kernel with additional noise added to it.
The goal is to recover the original (binary) image. We formulate the energy as outlined in \cite{raj-zabih-iccv-2005}.
The resulting interactions are pairwise and Figure \ref{deconvolution} (b) shows the solution obtained using RD, here the gray pixels are unlabeled. 
For comparison we also plotted the solution obtained when solving the the same problem as RD but with 
\cite{sontag-etal-uai-2008} (c) and TRW-S (d). 
For these methods there are substantial duality gaps.
In contrast (e) shows the solution obtained when forming super nodes with patches of size $3 \times 3$ and then applying TRW-S.
Here there is no duality gap, so the obtained solution is optimal. 

\begin{figure*}[htb]
\vspace{-5mm}
\centering
\def\iccvwidth{0.19\linewidth}%
\hfill
\subfloat[Noisy image]{%
\includegraphics[width=\iccvwidth]{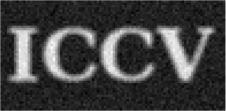}
}\hfill
\subfloat[RD (379.44)]{%
\includegraphics[width=\iccvwidth]{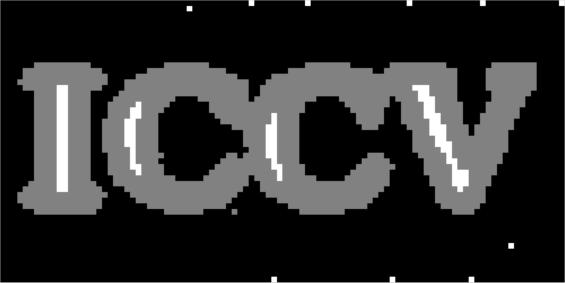}
}\hfill
\subfloat[MPLP (324.07)]{%
\includegraphics[width=\iccvwidth]{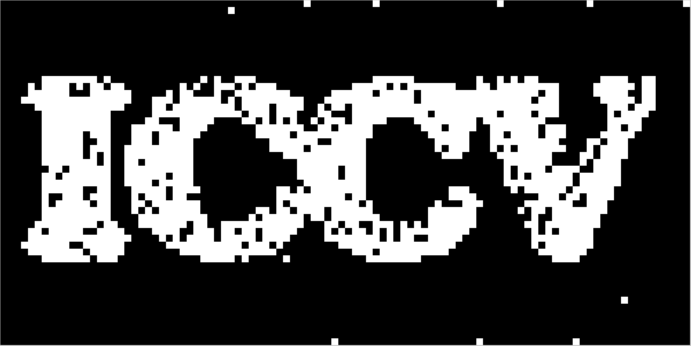}
}
\hfill
\subfloat[TRW-S (36.44)]{%
\includegraphics[width=\iccvwidth]{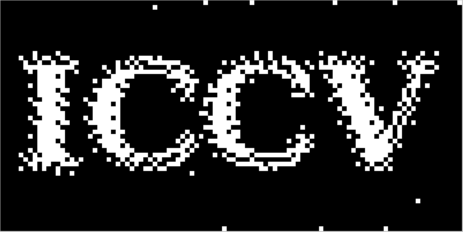}
}\hfill
\subfloat[TRW-S~patches (12.11)]{%
\includegraphics[width=\iccvwidth]{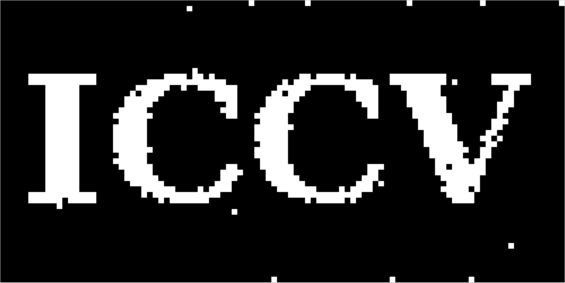}
}\hfill

\caption{\small Binary deconvolution results (with energy). 
(a) Noisy image. 
(b) Gray pixels are unlabeled.  Duality gap:  $477.78$ (unlabeled set to 0)  
(c) Duality gap: $422.40$, maximum 10,000 clusterings.  
(d) Duality gap: $134.77$. 
(e) Duality gap: $10^{-14}$.
}
\label{deconvolution}
\vspace{-4mm}
\end{figure*}

\def\cones{\texttt{Cones}}

\subsection{Stereo}
\setlength{\tabcolsep}{4pt} 
\begin{figure*}[htb]
	 \begin{centering}

		\subfloat[$\ell_1$ regularization.]{
		 \small

		 \begin{tabular}{l@{}cccr}
		\toprule
		 & Energy & \hspace{-2mm} Lower bound & Relative gap & Time (s) \\ 
		 \midrule
		 Our & \hspace{-2mm}$1.4558\cdot 10^{10}$ &  $1.4558 \cdot 10^{10}$ & $1.4471 \cdot 10^{-14}$ &   $315.3342$\\
		 RD &  \hspace{-2mm} $1.4637\cdot 10^{10}$ & $1.4518 \cdot 10^{10}$ & $9.3694 \cdot 10^{-3}$ & $1.9216$\\
		 Our/RD &   $0.9958$ &   $1.0019$ & $4.3825\cdot 10^{-13}$ & $180.6859$ \\
		 \bottomrule
		 \end{tabular}
		}
		\subfloat[$\ell_2$ regularization.]{
		 \small

		 \begin{tabular}{l@{}cccr}
		\toprule
		 & Energy & \hspace{-2mm} Lower bound & Relative gap & Time (s) \\ 
		 \midrule
		Our &  \hspace{-2mm}$1.3594 \cdot 10^{10}$ &  $1.3594 \cdot 10^{10}$ &  $2.0394 \cdot 10^{-14}$ &  
  $428.2557$ \\
		 RD &  \hspace{-2mm}$1.5165 \cdot 10^{10}$ &  $1.0484 \cdot 10^{10}$ &  $0.5756$ &  $4.6479$ \\
		Our/RD &  $0.9092$ &   $1.1652$ &    $6.0910\cdot 10^{-15}$ & $111.8597$ \\
		 \bottomrule
		 \end{tabular}
		}\vspace{-2mm}
	\captionof{table}{\small Averaged stereo results on $\cones$~sequence. Relative gap is defined as (Energy-Lower bound)/Lower bound. (a) For $\ell_1$ regularization RD left 24\% of the variables unlabeled. "Improve" lowered the average energy for RD to $1.4609\cdot 10^{10}$. (b) For $\ell_2$ regularization RD left 64\% of the variables unlabeled. "Improve" lowered the average energy for RD to $1.4392 \cdot 10^{10}$.}
	\label{tab:stereo_result}
	 \end{centering}
\vspace{-4mm}
\end{figure*}
\setlength{\tabcolsep}{6pt} 

In this section we optimize the energies occurring in Woodford \etal \cite{woodford2009}.
The goal is to find a dense disparity map for a given pair of images. 
The regularization of this method penalizes second order derivatives of the disparity map, either using a truncated $\ell_1$- or $\ell_2$-penalty.
The 2nd derivative is estimated from three consecutive disparity values (vertically or horizontally), thus resulting in triple interactions.

To solve this problem \cite{woodford2009} uses fusion moves \cite{lempitsky-etal-pami-2010} where proposals are fused together to lower the energy.
To compute the move \cite{woodford2009} first reduces the interactions (using auxiliary nodes) 
and applies Roof duality (RD) \cite{rother-etal-cvpr-2007}. 
In contrast we decompose the problem into patches of size $3\times3$, that contain entire triple interactions.
Since the interactions will occur in as many as three super nodes we weight these so that the energy does not change.

Table \ref{tab:stereo_result} shows the results for the \cones~dataset from \cite{middlebury2} when fusing "SegPln" proposals \cite{woodford2009}. Starting from a randomized disparity map we fuse all the proposals.
To ensure that each subproblem is identical for the two approaches, we feed the solution from RD into our approach
before running the next fusion.  We also tested the "improve" heuristic  \cite{rother-etal-cvpr-2007} which gave a reduction in duality gap for RD.
Running "probe" \cite{rother-etal-cvpr-2007} instead of improve is not feasible due to the large number of unlabeled variables.

We also compared the final energies when we ran the methods independent of each other (not feeding solutions into our approach). For $\ell_1$ regularization our solution has 0.82\% lower energy than that of RD with "improve" and for $\ell_2$ regularization our solution is 7.07\% lower than RD with "improve".

{\footnotesize
\bibliographystyle{ieee}
\bibliography{newlib}

\begin{thebibliography}{10}\itemsep=-1pt

\bibitem{Bacchus:AI02}
F.~Bacchus, X.~Chen, P.~van Beek, and T.~Walsh.
\newblock Binary vs. non-binary constraints.
\newblock {\em Artificial Intelligence}, 140(1/2):1--37, 2002.

\bibitem{BK:iccv03}
Y.~Boykov and V.~Kolmogorov.
\newblock Computing geodesics and minimal surfaces via graph cuts.
\newblock In {\em International Conference on Computer Vision (ICCV)}, 2003.

\bibitem{boykov-etal-pami-2001}
Y.~Boykov, O.~Veksler, and R.~Zabih.
\newblock Fast approximate energy minimization via graph cuts.
\newblock {\em IEEE Transations on Pattern Analysis and Machine Intelligence},
  23(11):1222--1239, 2001.

\bibitem{Pock:JMIV12}
K.~Bredies, T.~Pock, and B.~Wirth.
\newblock Convex relaxation of a class of vertex penalizing functionals.
\newblock {\em J. Math. Imaging and Vision}, 47(3):278--302, 2013.

\bibitem{ChanShen:01}
T.~Chan and J.~Shen.
\newblock Nontexture inpainting by curvature driven diffusion (cdd).
\newblock {\em Journal of Visual Communication and Image Representation},
  12:436--449, 2001.

\bibitem{DroskeRumpf:04}
M.~Droske and M.~Rumpf.
\newblock A level set formulation for {W}illmore flow.
\newblock {\em Interfaces and Free Boundaries}, 6:361--378, 2004.

\bibitem{elzehiry-grady-cvpr-2010}
N.~El-Zehiry and L.~Grady.
\newblock Fast global optimization of curvature.
\newblock In {\em Conf. Computer Vision and Pattern Recognition}, 2010.

\bibitem{felzenszwalb-veksler-cvpr-2010}
P.~F. Felzenszwalb and O.~Veksler.
\newblock Tiered scene labeling with dynamic programming.
\newblock In {\em IEEE Conf. on Computer Vision and Pattern Recognition}, 2010.

\bibitem{geman-geman-pami-1984}
S.~Geman and D.~Geman.
\newblock Stochastic relaxation, gibbs distributions, and the bayesian
  restoration of images.
\newblock {\em IEEE Transations on Pattern Analysis and Machine Intelligence},
  6(6):721--741, 1984.

\bibitem{Globerson:NIPS07}
A.~Globerson and T.~Jaakkola.
\newblock Fixing max-product: Convergent message passing algorithms for {MAP}
  {LP}-relaxations.
\newblock In {\em NIPS}, 2007.

\bibitem{ishikawa-pami-2003}
H.~Ishikawa.
\newblock Exact optimization for markov random fields with convex priors.
\newblock {\em IEEE Trans on Pattern Analysis and Machine Intelligence},
  25(10):1333 -- 1336, 2003.

\bibitem{kahl-strandmark-dam-2012}
F.~Kahl and P.~Strandmark.
\newblock Generalized roof duality.
\newblock {\em Discrete Applied Mathematics}, 160(16-17):2419--2434, 2012.

\bibitem{Kass:88}
M.~Kass, A.~Witkin, and D.~Terzolpoulos.
\newblock Snakes: Active contour models.
\newblock {\em Int. Journal of Computer Vision}, 1(4):321--331, 1988.

\bibitem{Koller:09}
D.~Koller and N.~Friedman.
\newblock {\em Probabilistic Graphical Models: Principles and Techniques}.
\newblock The MIT press, 2009.

\bibitem{kolmogorov-pami-2006}
V.~Kolmogorov.
\newblock Convergent tree-reweighted message passing for energy minimization.
\newblock {\em IEEE Transanctions on Pattern Analysis and Machine.
  Intelligence.}, 28:1568--1583, October 2006.

\bibitem{GTRWS:arXiv12}
V.~Kolmogorov and T.~Schoenemann.
\newblock Generalized sequential tree-reweighted message passing.
\newblock {\em arXiv:1205.6352}, 2012.

\bibitem{komodakis-etal-cvpr-2009}
N.~Komodakis and N.~Paragios.
\newblock Beyond pairwise energies: Efficient optimization for higher-order
  mrfs.
\newblock In {\em Conf. on Computer Vision and Pattern Recognition}, 2009.

\bibitem{Lellmann:SIAM11}
J.~Lellmann and C.~Schnorr.
\newblock Continuous multiclass labeling approaches and algorithms.
\newblock {\em SIAM Journal on Imaging Sciences}, 4:1049--1096, 2011.

\bibitem{lempitsky-etal-pami-2010}
V.~S. Lempitsky, C.~Rother, S.~Roth, and A.~Blake.
\newblock Fusion moves for markov random field optimization.
\newblock {\em IEEE Trans. Pattern Anal. Mach. Intell.}, 32(8):1392--1405,
  2010.

\bibitem{MasnouMorel:98}
S.~Masnou and J.~Morel.
\newblock Level-lines based disocclusion.
\newblock In {\em International Conference on Image Processing (ICIP)}, 1998.

\bibitem{Meltzer:UAI09}
T.~Meltzer, A.~Globerson, and Y.~Weiss.
\newblock Convergent message passing algorithms - a unifying view.
\newblock In {\em Conf. on Uncertainty in Artificial Intelligence}, 2009.

\bibitem{Nikolova:SIAM06}
M.~Nikolova, S.~Esedoglu, and T.~Chan.
\newblock Algorithms for finding global minimizers of image segmentation and
  denoising models.
\newblock {\em SIAM Journal of Applied Mathematics}, 66:1632--1648, 2006.

\bibitem{olsson:CVPR12}
C.~Olsson and Y.~Boykov.
\newblock Curvature-based regularization for surface approximation.
\newblock In {\em Conf. Computer Vision and Pattern Recognition}, 2012.

\bibitem{Pock:SIAM10}
T.~Pock, D.~Cremers, H.~Bischof, and A.~Chambolle.
\newblock Global solutions of variational models with convex regularization.
\newblock {\em SIAM Journal on Imaging Sciences}, 3:1122--1145, 2010.

\bibitem{raj-zabih-iccv-2005}
A.~Raj and R.~Zabih.
\newblock A graph cut algorithm for generalized image deconvolution.
\newblock In {\em International Conference of Computer vision (ICCV)}, 2005.

\bibitem{rother-etal-cvpr-2007}
C.~Rother, V.~Kolmogorov, V.~S. Lempitsky, and M.~Szummer.
\newblock Optimizing binary mrfs via extended roof duality.
\newblock In {\em Conf. Computer Vision and Pattern Recognition}, 2007.

\bibitem{middlebury2}
D.~Scharstein and R.~Szeliski.
\newblock High-accuracy stereo depth maps using structured light.
\newblock In {\em Conf. Computer Vision and Pattern Recognition}, 2003.

\bibitem{schoenemann-etal-ijcv-2012}
T.~Schoenemann, F.~Kahl, S.~Masnou, and D.~Cremers.
\newblock A linear framework for region-based image segmentation and inpainting
  involving curvature penalization.
\newblock {\em Int. Journal of Computer Vision}, 2012.

\bibitem{shekhovtsov-etal-dagm-2012}
A.~Shekhovtsov, P.~Kohli, and C.~Rother.
\newblock Curvature prior for {MRF}-based segmentation and shape inpaint.
\newblock In {\em arXiv: 1109.1480v1, 2011, also DAGM}, 2012.

\bibitem{sontag-etal-uai-2008}
D.~Sontag, T.~Meltzer, A.~Globerson, T.~Jaakkola, and Y.~Weiss.
\newblock Tightening lp relaxations for map using message passing.
\newblock In {\em UAI}, 2008.

\bibitem{strandmark-kahl-emmcvpr-2011}
P.~Strandmark and F.~Kahl.
\newblock Curvature regularization for curves and surfaces in a global
  optimization framework.
\newblock In {\em EMMCVPR}, pages 205--218, 2011.

\bibitem{Sullivan:thesis92}
J.~Sullivan.
\newblock A crystalline approximation theorem for hypersurfaces, phd thesis.
\newblock 1992.

\bibitem{Tarlow:AISTATS10}
D.~Tarlow, I.~Givoni, and R.~Zemel.
\newblock {HOP-MAP}: efficient message passing with higher order potentials.
\newblock In {\em AISTATS}, 2010.

\bibitem{Wainwright:trw_max}
M.~Wainwright, T.~Jaakkola, and A.~Willsky.
\newblock {MAP} estimation via agreement on (hyper)trees: Message-passing and
  linear-programming approaches.
\newblock {\em IEEE Trans. on Information Theory}, 51(11):3697--3717, Nov.
  2005.

\bibitem{Werner:PAMI07}
T.~Werner.
\newblock A linear programming approach to max-sum problem: A review.
\newblock {\em IEEE Transations on Pattern Analysis and Machine Intelligence},
  29(7):1165--1179, 2007.

\bibitem{Werner:PAMI10}
T.~Werner.
\newblock Revisiting the linear programming relaxation approach to {G}ibbs
  energy minimization and weighted constraint satisfaction.
\newblock {\em IEEE Transations on Pattern Analysis and Machine Intelligence},
  32(8):1474--1488, 2010.

\bibitem{woodford2009}
O.~Woodford, P.~Torr, I.~Reid, and A.~Fitzgibbon.
\newblock Global stereo reconstruction under second order smoothness priors.
\newblock {\em IEEE Transactions on Pattern Analysis and Machine Intelligence},
  31(12):2115--2128, 2009.

\bibitem{yuan:eccv10}
J.~Yuan, E.~Bae, X.-C. Tai, and Y.~Boykov.
\newblock A continuous max-flow approach to potts model.
\newblock In {\em European Conference on Computer Vision (ECCV)}, 2010.

\end{thebibliography}
}

\appendix
\section{LP relaxation}\label{app:lprelaxation}
In this appendix we relate our proposed partial enumeration approach to other methods, that optimize higher order energy factors,
by analyzing the particular LP relaxation being solved. 

\subsection{Consistent Labelings}
We consider energy functions of the form
\begin{equation}
E(\bx)=\sum_{\alpha\in\calC} E_\alpha(\bx_\alpha)
\label{eq:E}
\end{equation}
where letter $\alpha$ denotes a {\em cluster} (i.e.\ a set of variables in $V$) , $\bx_\alpha$ is the restriction of $\bx$ to $\alpha$,
and $\calC$ is some set of clusters. We assume that variable  $x_i$ for $i\in V$ takes values
in some discrete set of labels $\calL_i$. Where appropriate, we will treat $\alpha$ as an {\em ordered} sequence of nodes
(e.g.\ with respect to some total order on $V$).
For a cluster $\alpha=(i_1,\ldots,i_k)$
we denote $\calL_{\alpha}=\calL_{i_1}\times\ldots\times\calL_{i_k}$. We have, in particular, $\bx\in\calL_V$.

Let us select another set of clusters $\calV$ with the following property:
for every $\alpha\in\calC$ there exists $\beta\in\calV$ such that $\alpha\subseteq\beta$.
Clearly, we can equivalently rewrite energy~\eqref{eq:E} as an energy with unary and pairwise terms:
\begin{equation}
\hat E(\bX)=\sum_{\alpha\in\calV} U_\alpha(X_\alpha) + \sum_{\{\alpha,\beta\}\in\calE} P_{\alpha\beta}(X_\alpha,X_\beta)
\label{eq:hatE}
\end{equation}
Here $X_\alpha$ is a discrete variable that takes values in $\calL_\alpha$.
The pairwise potential $P_{\alpha\beta}$ is defined as follows:
$P_{\alpha\beta}(X_\alpha,X_\beta)=0$ if variables $X_\alpha$ and $X_\beta$ agree on the overlap $\alpha\cap\beta$,
and $P_{\alpha\beta}(X_\alpha,X_\beta)=+\infty$ otherwise. It remains to specify how to choose the set of edges
$\calE$. One possibility would be to select all pairs $\{\alpha,\beta\}$ such that $\alpha,\beta\in\calV$ and $\alpha\cap\beta\ne\varnothing$.
However, in some cases we may be able to choose a smaller set.
\begin{proposition}\label{prop:consistency}
Suppose that graph $(\calV,\calE)$ satisfies the following property: 
\begin{itemize}
\item For every pair of distinct 
clusters $\alpha,\beta\in\calV$ with $\gamma=\alpha\cap\beta\ne\varnothing$
the subgraph $(\calV_{\gamma},\calE_{\gamma})$ of $(\calV,\calE$) induced
by the set of nodes $\calV_{\gamma}=\{\alpha\in\calV\:|\:\gamma\subseteq\alpha\}$
is connected.
\end{itemize}
Then labeling $\bX$ is consistent\footnote{We say that $\bX$ is consistent
if there exists a labeling $\bx\in \calL_V$ such that $X_\alpha=\bx_\alpha$ for all $\alpha\in\calV$} iff
$\sum_{\{\alpha,\beta\}\in\calE} P_{\alpha\beta}(X_\alpha,X_\beta)=0$.
\label{prop:partialE}
\end{proposition}
\begin{proof}
One direction is trivial: if $\bX$ is consistent then each pairwise term $P_{\alpha\beta}(X_\alpha,X_\beta)$
is zero. 

Suppose that $P_{\alpha\beta}(X_\alpha,X_\beta)=0$ for all $\{\alpha,\beta\}\in\calE$.
Let us define labeling $\bx\in\calL_V$ as follows: for a node $p\in V$
select cluster $\alpha\in\calV$ with $p\in\alpha$ and set $x_p:=(X_\alpha)_p$.
We need to show that this definition does not depend on the exact choice of $\alpha$.
Consider two distinct clusters $\alpha,\beta\in\calV$ with $p\in\alpha\cap\beta$.
By the assumption of the proposition, nodes $\alpha,\beta$ are connected by a path
 $\alpha_0,\ldots,\alpha_k$ in graph $(\calV_\gamma,\calE_\gamma)$ where $\gamma=\alpha\cap\beta$, $p\in\gamma$.
For each $i$ we have $P_{\alpha_i\alpha_{i+1}}(X_{\alpha_i},X_{\alpha_{i+1}})=0$,
and so labelings $X_{\alpha_i}$ and $X_{\alpha_{i+1}}$ agree on $p\in\alpha_i\cap\alpha_{i+1}$.
An induction argument then shows that $X_\alpha$ and $X_\beta$ agree on $p$, thus proving the claim.

Showing that the constructed labeling $\bx$ is consistent with $X_\alpha$ for each $\alpha\in\calV$
is now straightforward.
%
%
\end{proof}

\begin{figure}[htb]
\begin{center}
\begin{pgfpicture}{0cm}{0cm}{5cm}{3cm}
\pgfputat{\pgfxy(0,0.35)}{\includegraphics[width=4.8cm]{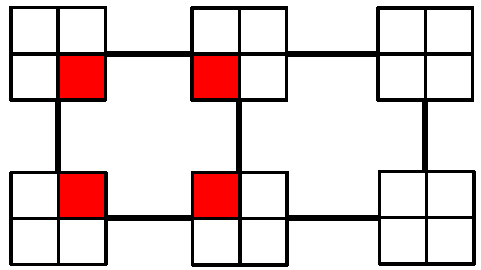}}
	\pgfputat{\pgfxy(0.5,3.15)}{$\alpha$}
	\pgfputat{\pgfxy(2.2,3.15)}{$\alpha_1$}
	\pgfputat{\pgfxy(2.2,0)}{$\beta$}
\end{pgfpicture}	
\end{center}
\caption{Example of proposition~\ref{prop:consistency} for the case of $2\times2$ patches. The intersection $\alpha \cap \beta$ is the red pixel which is also contained in $\alpha_1$.
Therefore no diagonal consistency edge between $\alpha$ and $\beta$ is needed.}
\end{figure}

\subsection{Analysis of LP relaxations}
When we apply method like TRW-S to energy~\eqref{eq:hatE}, we essentially solve a higher-order relaxation of the original energy~\eqref{eq:E}.
Many methods have been proposed in the literature for solving higher-order relaxations, 
e.g.~\cite{sontag-etal-uai-2008,komodakis-etal-cvpr-2009,Meltzer:UAI09,Werner:PAMI10,GTRWS:arXiv12} to name just a few.
To understand the relation to these methods, we will analyze which specific relaxation is solved by our approach.
We will then argue that the complexity of message passing in our scheme roughly matches that
of other techniques that solve a similar relaxation. 
\footnote{\label{foot:specialized} Message
passing techniques require the minimization of expressions of the form
 $E_\alpha(\bx_\alpha)+\ldots$ where dots denote lower-order factors.
Here we assume that this expression is minimized by going through all possible labelings $\bx_\alpha$.
This would hold if, for example, $E_\alpha(\cdot)$ is represented by a table (which is the case
with curvature). Some terms $E_\alpha(\cdot)$ used in practice have a special structure
that allow more efficient computations; in this case other techniques may have a better
complexity. One example is {\em cardinality-based potentials}~\cite{Tarlow:AISTATS10} which can have a very high order.}
In practice, the choice of the optimization method is often motivated by the ease of implementation;
we believe that our scheme has an advantage in this respect, and thus may be preferred by practitioners.

\myparagraph{Family of LP relaxations}
We use the framework of Werner~\cite{Werner:PAMI10} who described a large family of LP relaxations.
Each relaxation is specified by two sets, $\calF$ and $J$. 
Set $\calF$ contains clusters $\alpha\subseteq V$, with $\calC\subseteq\calF$. For each $\alpha\in\calF$ and for each possible labeling $\bx_\alpha\in\calL_\alpha$
an indicator variable $\tau_{\alpha}(\bx_\alpha)\in\{0,1\}$ is introduced; the integrality constraint is then relaxed to $\tau_{\alpha}(\bx_\alpha)\in[0,1]$.
Set $J$ contains pairs $(\alpha,\beta)$ with $\alpha,\beta\in\calF$, $\beta\subset\alpha$;
this means that $(\calF,J)$ is a directed acyclic graph. For each edge $(\alpha,\beta)\in J$
we add a {\em consistency}, or a {\em marginalization} constraint between $\alpha$ and $\beta$.
The resulting relaxation is given by
\begin{subequations}\label{eq:WernerLP}
\begin{eqnarray}
 &&\hspace{-90pt}\min \;\;\; \sum_{\alpha\in\calC}\sum_{\bx_\alpha}E_\alpha(\bx_\alpha)\tau_{\alpha}(\bx_\alpha) \label{eq:WernerLP:a}\\
\mbox{s.t.~~~~~}\sum_{\bx_\alpha} \tau_\alpha(\bx_\alpha)&\!\!\!\!=\!\!\!\!&1 \hspace{35pt}\forall \alpha\in\calF \label{eq:WernerLP:b}\\
\sum_{\bx_{\alpha}:\bx_\alpha\sim\bx_\beta}\!\!\!\!\!\! \tau_\alpha(\bx_\alpha)&\!\!\!\!=\!\!\!\!&\tau_\beta(\bx_\beta)\hspace{11pt}\forall(\alpha,\beta)\in J,\forall \bx_\beta\label{eq:WernerLP:c}\\
 \tau_\alpha(\bx_\alpha)&\!\!\!\!\ge\!\!\!\!&0 \hspace{36pt}\forall \alpha\in\calF,\forall\bx_\alpha \label{eq:WernerLP:d}
\end{eqnarray}
\end{subequations}
where notation $\bx_\alpha\sim\bx_\beta$ means that labelings $\bx_\alpha$ and $\bx_\beta$ are consistent on the overlap area.

Adding more edges to $J$ gives more constraints and thus leads to the same or tighter relaxation.
The simplest choice is to set $J=\{(\alpha,\{i\})\:|\:\alpha\in \calF,i\in\alpha\}$;
in \cite{GTRWS:arXiv12} this is called a {\em relaxation with singleton separators}.
We will show next that our approach uses a larger set $J$;
experimental results in \cite{GTRWS:arXiv12} and in our paper confirm that
this gives a tighter relaxation.

\paragraph{Partial enumeration}
The main optimization technique studied in this paper is to convert energy~\eqref{eq:E} to energy~\eqref{eq:hatE} and then
apply TRW-S algorithm~\cite{kolmogorov-pami-2006} to the latter.
It is known \cite{Wainwright:trw_max} that TRW techniques for pairwise energies attempt to solve a certain LP relaxation
 known as {\em Schlesinger's LP}~\cite{Werner:PAMI07}.
In order to understand the relation to previous techniques, we will formulate
the resulting relaxation in terms of the {\em original} energy~\eqref{eq:E}.
\begin{definition}
 Consider a high-order energy~\eqref{eq:E} with the set of factors $\calC$.
The {\em maximal LP relaxation} of~\eqref{eq:E} is the relaxation~\eqref{eq:WernerLP} 
with $\calF=\{\alpha\:|\:\alpha\subseteq\hat\alpha\mbox{ for some }\hat\alpha\in\calC\}$
and $J=\{(\alpha,\beta)\:|\:\alpha,\beta\in\calF,\beta\subset\alpha\}$.
\end{definition}
\begin{theorem}
Suppose that $\calV=\calC$ and graph $(\calV,\calE)$ satisfies the precondition of proposition~\ref{prop:consistency}.
Then Schlesinger's LP relaxation for energy~\eqref{eq:hatE} is equivalent to the maximal LP relaxation
of energy~\eqref{eq:E}.
\label{th:Jmain}
\end{theorem}
We will also prove the following result.
\begin{theorem}
Suppose that graph $(\calF,J)$ satisfies the following for any non-empty subset $\beta\subseteq \alpha$ where $\alpha\in\calF$:
\begin{itemize}
\item The subgraph $(\calF_\beta , J_\beta )$ of $(\calF, J)$ induced by the set of clusters 
$\calF_\beta = \{\alpha \in \calF\: | \:\beta \subseteq \alpha\}$
is connected.
\end{itemize}
Then relaxation~\eqref{eq:WernerLP} with graph $(\calF,J)$ is equivalent to the maximal LP relaxation of energy~\eqref{eq:E}.
\label{th:Jmain'}
\end{theorem}
Proofs of these theorems are   given in sections \ref{sec:proof:th:Jmain} and  \ref{sec:proof:th:Jmain'}. 

We say that graph $(\calF,J)$ is {\em maximal}
if it satisfies the following: (i) if $\alpha\in\calF$ then any non-empty subset of $\alpha$ is also in $\calF$;
(ii) if $\alpha,\beta\in\calF$ and $\beta\subseteq\alpha$ then $(\alpha,\beta)\in J$.
We believe that using maximal graphs may often be beneficial in practice:
they would give tighter relaxations compared to non-maximal graphs,
and solving them should not be much more difficult.
Indeed, if we use message passing algorithms then we need to send messages 
along edges $(\alpha,\beta)\in J$. A naive implementation of that takes $O(|\calL_\alpha|)$ time,
and so the complexity of message passing is mainly determined by the size of
maximal clusters in $\calF$. It remains to note that any non-maximal graph $(\calF,J)$ can be extended to a maximal one
without changing the set of maximal clusters.

Note that in our scheme sending a message from $\alpha$ to $\beta$ for $\{\alpha,\beta\}\in\calE$
takes $O(|\calL_\alpha|+|\calL_\beta|)$ time, if we use the technique described in Section 3.2 of our main paper.
Therefore, the complexity of applying TRW-S to energy~\eqref{eq:hatE}
should rougly match the complexity of other message passing techniques
that would solve an equivalent relaxation for the original energy~\eqref{eq:E}.
(One exception is when we have specialized high-order terms, as discussed in footnote~\footref{foot:specialized}).


\myparagraph{Approach in \cite{komodakis-etal-cvpr-2009}}
To conclude this section, we will discuss the relaxation solved in~\cite{komodakis-etal-cvpr-2009}.
Their work also uses square patches, and thus bears some resemblance to our approach.

Eq. (2)-(5) in \cite{komodakis-etal-cvpr-2009} say that they solve a relaxation
with singleton separators for some set $\calF$, i.e.\ $J=\{(\alpha,\{i\})\:|\:\alpha\in \calF,i\in\alpha\}$.
\cite{komodakis-etal-cvpr-2009} proposes two ways for choosing set $\calF$ for a grid graph: (a) as patches of a fixed size $K\times K$;
(b) as horizontal and vertical stripes of sizes $K\times N$ and $M\times K$ respectively.
In the first case set $J$ is strictly smaller than in our relaxation (if we take $\calV$ as the set of patches $K\times K$),
and the corresponding relaxation is weaker.
The second case is more diffucult to analyze, 
but we conjecture that the resulting relaxation would still weaker than ours.
First, we believe that the relaxation would not change if we ``break'' stripes into
patches while enforcing consistency between adjacent patches.
Now for each patch $\alpha$ of size $K\times K$ we have two sets of indicator variables, $\tau_{\alpha^{\tt horz}}(\cdot)$
and $\tau_{\alpha^{\tt vert}}(\cdot)$. While these variables have strong connections
to the appropriate neighbors of the same type (horizontal/vertical), the agreement between the two 
is enforced only loosely: they are just required to have the same unary marginals.
Intuitively, we believe that this would be weaker than the relaxation described in Theorem~\ref{th:Jmain}.


\subsubsection{Proof of Theorem~\ref{th:Jmain}}\label{sec:proof:th:Jmain}


First, let us write down the Schlesinger's LP for energy~\eqref{eq:hatE}.
It uses variables $\hat\tau_\alpha(\bx_\alpha)$ for $\alpha\in\calV$, $\bx_\alpha\in\calL_\alpha$
and variables  $\hat\tau_{\alpha\beta}(\bx_\alpha,\bx_\beta)$ for $\{\alpha,\beta\}\in\calE$, $(\bx_\alpha,\bx_\beta)\in\calL_\alpha\times\calL_\beta$
where $\hat\tau_{\alpha\beta}(\bx_\alpha,\bx_\beta)$ and $\hat\tau_{\beta\alpha}(\bx_\beta,\bx_\alpha)$ are treated as the same variable.
If $\bx_\alpha\nsim\bx_\beta$ then variable $\hat\tau_{\alpha\beta}(\bx_\alpha,\bx_\beta)$ will be zero at the optimum,
since the associated cost is infinite. Thus, we can eliminate such variables from the formulation. We get the following LP:
\begin{subequations}\label{eq:Xrelax}
\begin{eqnarray}
 &&\hspace{-90pt}\min \;\;\; \sum_{\alpha\in\calV}\sum_{\bx_\alpha}E_\alpha(\bx_\alpha)\hat\tau_{\alpha}(\bx_\alpha) \label{eq:Xrelax:a}\\
\mbox{s.t.~~~~~}\sum_{\bx_\alpha} \hat\tau_\alpha(\bx_\alpha)&\!\!\!\!=\!\!\!\!&1 \hspace{35pt}\forall \alpha\in\calV \label{eq:Xrelax:b}\\
\sum_{\bx_{\beta}:\bx_\beta\sim\bx_\alpha}\!\!\!\!\!\! \hat\tau_{\alpha\beta}(\bx_\alpha,\bx_\beta)&\!\!\!\!=\!\!\!\!&\hat\tau_\alpha(\bx_\alpha)\hspace{11pt}\forall\{\alpha,\beta\}\in \calE,\forall \bx_\alpha\quad\quad\label{eq:Xrelax:c}\\
 \hat\tau_\alpha(\bx_\alpha)&\!\!\!\!\ge\!\!\!\!&0                  \hspace{15pt}\forall \alpha\in\calV,\forall\bx_\alpha \label{eq:Xrelax:d} \\
 \hat\tau_{\alpha\beta}(\bx_\alpha,\bx_\beta)&\!\!\!\!\ge\!\!\!\!&0 \hspace{15pt}\forall \{\alpha,\beta\}\in\calE,\forall\bx_\alpha,\bx_\beta \label{eq:Xrelax:e} \\
 & & \hspace{70pt} \mbox{s.t.~} \bx_\alpha\sim\bx_\beta \nonumber
\end{eqnarray}
\end{subequations}
Our goal is to show that this LP is equivalent to the LP~\eqref{eq:WernerLP} with the graph $(\calF,J)$ defined in Theorem~\ref{th:Jmain}.
Let $\hat\Omega$ and $\Omega$ be the feasible sets of the two LPs. We will prove the claim by showing that there exist cost-preserving
mappings $\hat\Omega\rightarrow\Omega$ and $\Omega\rightarrow\hat\Omega$.




\paragraph{Mapping $\hat\Omega\rightarrow\Omega$}
Given a vector $\hat\tau\in\hat\Omega$, we define vector $\tau$ as follows.
Consider $\gamma\in\calF$. By the definition of $\calF$, there exists $\alpha\in\calV$ with $\gamma\subseteq\alpha$. We set
\begin{equation}
\tau_\gamma(\bx_\gamma)=\sum_{\bx_{\alpha}:\bx_\alpha\sim\bx_\gamma}\!\!\!\!\!\! \hat\tau_{\alpha}(\bx_\alpha)
\label{eq:FASFNAOSIF}
\end{equation}
Let us show that this definition does not depend on the choice of $\alpha$. Suppose there are
two clusters $\alpha,\beta\in\calV$ with $\gamma\subseteq\alpha\cap\beta$. 
We consider three cases:

\myparagraph{Case 1: $\gamma=\alpha\cap\beta$, $\{\alpha,\beta\}\in\calE$.} 
Using condition~\eqref{eq:Xrelax:c} for pairs $(\alpha,\beta)$ and $(\beta,\alpha)$, we obtain the desired result:
\begin{eqnarray*}
\sum_{\bx_{\alpha}:\bx_\alpha\sim\bx_\gamma}\!\!\!\!\!\! \hat\tau_{\alpha}(\bx_\alpha) 
= \sum_{\bx_{\alpha}:\bx_\alpha\sim\bx_\gamma}\;\;\sum_{\bx_{\beta}:\bx_\beta\sim\bx_\alpha}\!\!\!\!\!\! \hat\tau_{\alpha\beta}(\bx_\alpha,\bx_\beta) \hspace{10pt}\\
 = \sum_{\bx_{\alpha}:\bx_\alpha\sim\bx_\gamma}\;\;\sum_{\bx_{\beta}:\bx_\beta\sim\bx_\gamma}\!\!\!\!\!\! \hat\tau_{\alpha\beta}(\bx_\alpha,\bx_\beta) \hspace{30pt}\\
\hspace{20pt} = \sum_{\bx_{\beta}:\bx_\beta\sim\bx_\gamma}\;\;\sum_{\bx_{\alpha}:\bx_\alpha\sim\bx_\beta}\!\!\!\!\!\! \hat\tau_{\alpha\beta}(\bx_\alpha,\bx_\beta) 
= \sum_{\bx_{\beta}:\bx_\beta\sim\bx_\gamma}\!\!\!\!\!\! \hat\tau_{\beta}(\bx_\beta) 
\end{eqnarray*}

\myparagraph{Case 2: $\gamma\subset\alpha\cap\beta$, $\{\alpha,\beta\}\in\calE$.} 
Denote $\gamma'=\alpha\cap\beta$. Using the result that we just proved we obtain
\begin{eqnarray*}
\sum_{\bx_{\alpha}:\bx_\alpha\sim\bx_\gamma}\!\!\!\!\!\! \hat\tau_{\alpha}(\bx_\alpha) 
= \sum_{\bx_{\gamma'}:\bx_{\gamma'}\sim\bx_\gamma}\;\;\sum_{\bx_{\alpha}:\bx_\alpha\sim\bx_{\gamma'}}\!\!\!\!\!\! \hat\tau_{\alpha}(\bx_\alpha) 
\hspace{35pt} \\
\hspace{35pt}= \sum_{\bx_{\gamma'}:\bx_{\gamma'}\sim\bx_\gamma}\;\;\sum_{\bx_{\beta}:\bx_\beta\sim\bx_{\gamma'}}\!\!\!\!\!\! \hat\tau_{\beta}(\bx_\beta) 
= \sum_{\bx_{\beta}:\bx_\beta\sim\bx_\gamma}\!\!\!\!\!\! \hat\tau_{\beta}(\bx_\beta) 
\end{eqnarray*}

\myparagraph{Case 3: $\{\alpha,\beta\}\notin\calE$.} 
By the assumption of Proposition~\ref{prop:partialE}, nodes $\alpha$ and $\beta$ are connected
by a path $\alpha_0,\ldots,\alpha_k$ in graph $(\calV_{\alpha\cap\beta},\calE_{\alpha\cap\beta})$.
Note, $\gamma\subseteq\alpha\cap\beta\subseteq\alpha_i$ for each $i\in[0,k]$.
As proved above, for each $i\in[0,k-1]$ there holds
\begin{eqnarray*}
\sum_{\bx_{\alpha_{i}}:\bx_{\alpha_i}\sim\bx_\gamma}\!\!\!\!\!\! \hat\tau_{\alpha_{i}}(\bx_{\alpha_{i}}) 
=\sum_{\bx_{\alpha_{i+1}}:\bx_{\alpha_{i+1}}\sim\bx_\gamma}\!\!\!\!\!\! \hat\tau_{\alpha_{i+1}}(\bx_{\alpha_{i+1}}) 
\end{eqnarray*}
Using an induction argument, we obtain the desired result.

We proved that~\eqref{eq:FASFNAOSIF} is a valid definition that does not depend on the choice of $\alpha$.
Showing that obtained vector $\tau$ satisfies~\eqref{eq:WernerLP:c} for each $(\alpha,\beta)\in J$ is straightforward:
in the definition~\eqref{eq:FASFNAOSIF} for $\tau_\alpha(\bx_\alpha)$ and $\tau_\beta(\bx_\beta)$
we need to select the same cluster $\hat\alpha\in\calV$ with $\beta\subset\alpha\subseteq\hat\alpha$,
then~\eqref{eq:WernerLP:c} easily follows. Condition~\eqref{eq:Xrelax:b} implies~\eqref{eq:WernerLP:b}
for all $\alpha\in\calV$; combining this with~\eqref{eq:WernerLP:c}
gives condition~\eqref{eq:WernerLP:b} for all $\alpha\in\calF$. We proved that $\tau\in\Omega$.

\paragraph{Mapping $\Omega\rightarrow\hat\Omega$}
Consider vector $\tau\in\Omega$. We define $\hat\tau_\alpha(\bx_\alpha)=\tau_\alpha(\bx_\alpha)$ for clusters $\alpha\in\calV$ and
labelings $\bx_\alpha$.
For each edge $\{\alpha,\beta\}\in\calE$ and labelings $\bx_\alpha\sim\bx_\beta$ we define
\begin{equation}
\hat\tau_{\alpha\beta}(\bx_\alpha,\bx_\beta)=
\frac{\tau_\alpha(\bx_\alpha)\tau_\beta(\bx_\beta)}{\tau_\gamma(\bx_\gamma)} 
\end{equation}
where $\gamma=\alpha\cap\beta$; if $\tau_\gamma(\bx_\gamma)=0$ then we define $\hat\tau_{\alpha\beta}(\bx_\alpha,\bx_\beta)=0$ instead.

Let us show that~\eqref{eq:Xrelax:c} holds for a pair $\{\alpha,\beta\}\in\calE$ and a fixed labeling $\bx_\alpha$.
Denote $\gamma=\alpha\cap\beta$, and let $\bx_\gamma$ be the restriction of $\bx_\alpha$ to $\gamma$.
If $\tau_\gamma(\bx_\gamma)=0$ then from~\eqref{eq:WernerLP:c},\eqref{eq:WernerLP:d} we have $\tau_\alpha(\bx_\alpha)=0$, and so
both sides of~\eqref{eq:Xrelax:c} are zeros.
Otherwise we can write
\begin{eqnarray*}
\sum_{\bx_{\beta}:\bx_\beta\sim\bx_\alpha}\!\!\!\!\!\! \hat\tau_{\alpha\beta}(\bx_\alpha,\bx_\beta)
=\frac{\tau_\alpha(\bx_\alpha)}{\tau_\gamma(\bx_\gamma)}\sum_{\bx_{\beta}:\bx_\beta\sim\bx_\alpha}\!\!\!\!\!\! \tau_{\beta}(\bx_\beta) 
\hspace{50pt}\\
\hspace{10pt}=\frac{\tau_\alpha(\bx_\alpha)}{\tau_\gamma(\bx_\gamma)}\sum_{\bx_{\beta}:\bx_\beta\sim\bx_\gamma}\!\!\!\!\!\! \tau_{\beta}(\bx_\beta) 
=\frac{\tau_\alpha(\bx_\alpha)}{\tau_\gamma(\bx_\gamma)}\cdot \tau_{\gamma}(\bx_\gamma)
=\hat\tau_\alpha(\bx_\alpha)
\end{eqnarray*}
where we used condition~\eqref{eq:WernerLP:c}. We proved that $\hat\tau\in\hat\Omega$.

We finished the construction of mappings 
 $\hat\Omega\rightarrow\Omega$ and $\Omega\rightarrow\hat\Omega$.
In both cases we have $\hat\tau_\alpha(\bx_\alpha)=\tau_\alpha(\bx_\alpha)$ for $\alpha\in\calV$,
and therefore the mappings are cost-preserving.

\subsubsection{Proof of Theorem~\ref{th:Jmain'}}\label{sec:proof:th:Jmain'}
Denote $\hat \calF=\{\alpha\:|\:\alpha\subseteq\hat\alpha\mbox{ for some }\hat\alpha\in\calF\}$
and $\hat J=\{(\alpha,\beta)\:|\:\alpha,\beta\in\calF,\beta\subset\alpha\}$.
Let $\tau$ be a feasible vector of relaxation~\eqref{eq:WernerLP} with the graph $(\calF,J)$.
It suffices to show that such vector can be extended to a feasible vector $\hat\tau$ of 
relaxation~\eqref{eq:WernerLP} with the graph $(\hat\calF,\hat J)$.

Consider $\gamma\in\hat\calF$. By the definition of $\hat\calF$,
there exists $\alpha\in\calF$ with $\gamma\subseteq\alpha$. We set
\begin{equation}
\hat \tau_\gamma(\bx_\gamma)=\sum_{\bx_\alpha:\bx_\alpha\sim\bx_\gamma}\hat \tau_\alpha(\bx_\alpha)
\label{eq:fagasdfahsdga}
\end{equation}
Let us show that this definition does not depend on the choice of $\alpha$.
Suppose that there exist two clusters $\alpha,\beta\in\calF$ with $\gamma\in\alpha\cap\beta$.
We need to show that
\begin{equation}
\sum_{\bx_\alpha:\bx_\alpha\sim\bx_\gamma}\hat \tau_\alpha(\bx_\alpha)
=\sum_{\bx_\beta:\bx_\beta\sim\bx_\gamma}\hat \tau_\beta(\bx_\beta)
\end{equation}
By the assumption of the theorem, nodes $\alpha$ and $\beta$ are connected in graph $(\calF_\gamma,J_\gamma)$.
It suffices to prove the claim in the case when $\alpha$ and $\beta$ are connected by a single edge;
the main claim will then follow by induction on the length of the path between $\alpha$ and $\beta$.

Assume that $(\alpha,\beta)\in J$ (the case $(\beta,\alpha)\in J$ is symmetric).
By the choice of $\alpha,\beta$ we have $\gamma\subseteq\beta\subset\alpha$.
Using that facts that $\tau$ is feasible and $(\alpha,\beta)\in J$,
we can write
\begin{eqnarray*}
\sum_{\bx_\beta:\bx_\beta\sim\bx_\gamma}\tau_\beta(\bx_\beta)
&=&
\sum_{\bx_\beta:\bx_\beta\sim\bx_\gamma}\sum_{\bx_\alpha:\bx_\alpha\sim\bx_\beta}\tau_\alpha(\bx_\alpha) \\
&=&
\sum_{\bx_\alpha:\bx_\beta\sim\bx_\gamma}\tau_\alpha(\bx_\alpha)
\end{eqnarray*}

We proved that eq.~\eqref{eq:fagasdfahsdga} gives a valid definition of vector $\hat\tau$
that does not depend on the choice of $\alpha$.  From this definition we obtain
that $\hat\tau_\alpha=\tau_\alpha$ for all $\alpha\in\calF$.
It is also straightforward to check that $\hat\tau$ satisfies marginalization constraint~\eqref{eq:WernerLP:c}
for any $(\alpha,\beta)\in\hat J$.


\section{Patch Cost Assignments}\label{app:patchcost}

\begin{figure*}
\begin{center}
\begin{tabular}{ccccccc}
\includegraphics[width=20mm]{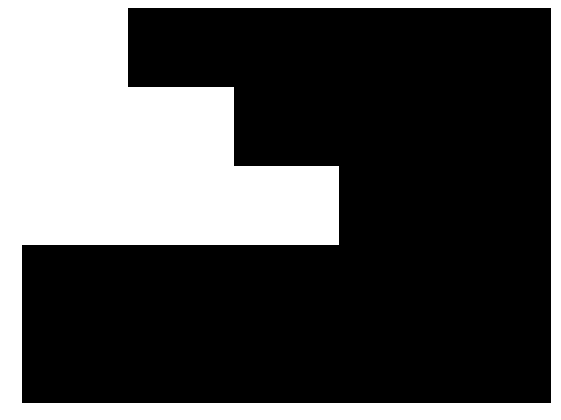} &
\includegraphics[width=20mm]{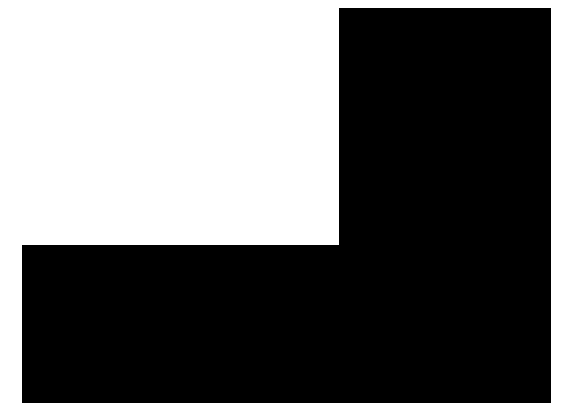} &
\includegraphics[width=20mm]{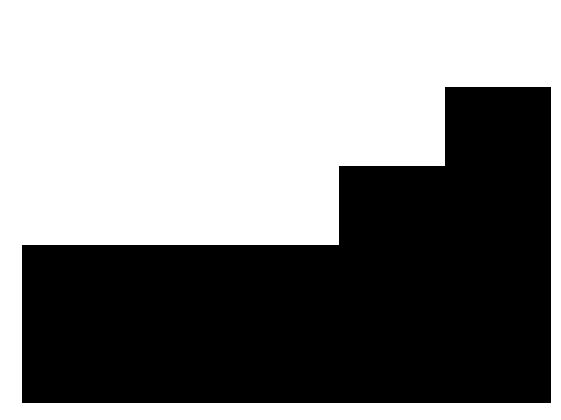} &
\includegraphics[width=20mm]{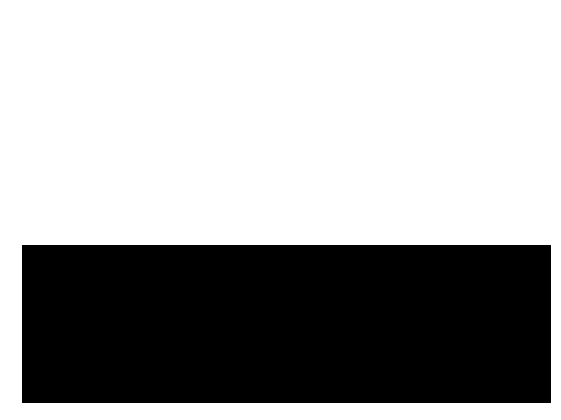} &
\includegraphics[width=20mm]{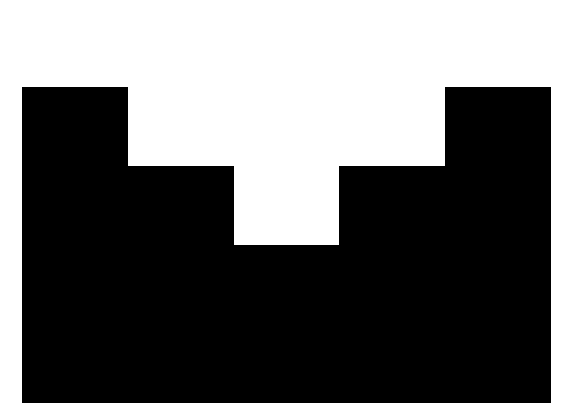} &
\includegraphics[width=20mm]{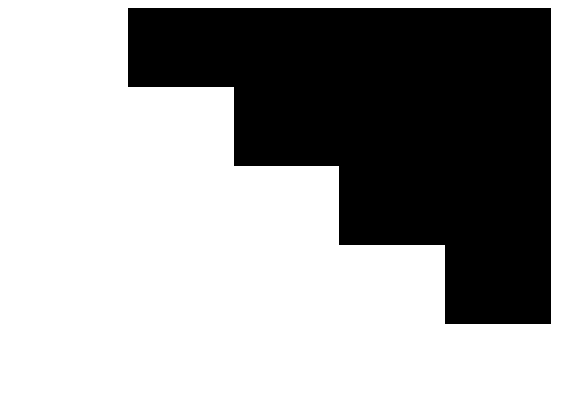} &
\includegraphics[width=20mm]{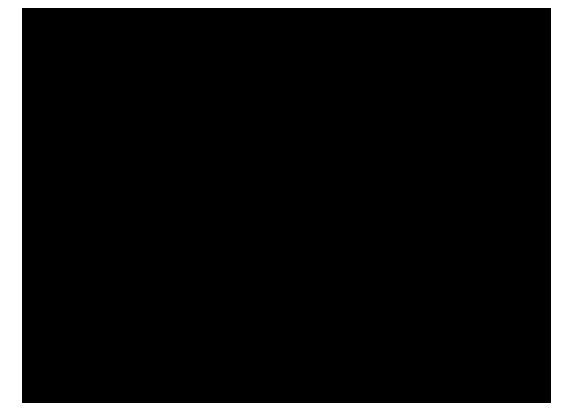} 
\end{tabular}
\caption{Seven of the $5\times 5$ windows used for computing patch costs for $3 \times 3$ curvature.
(The rest are obtained as rotations, reflections of the patches and inversions of the pixel assignments).}
\label{basewindows}
\begin{tabular}{ccccccc}
\includegraphics[width=20mm]{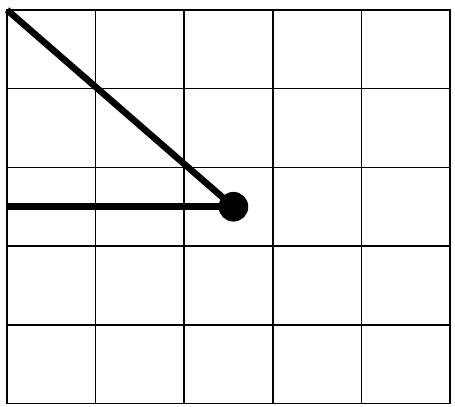} &
\includegraphics[width=20mm]{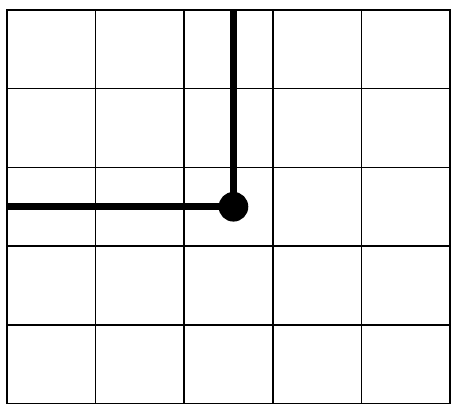} &
\includegraphics[width=20mm]{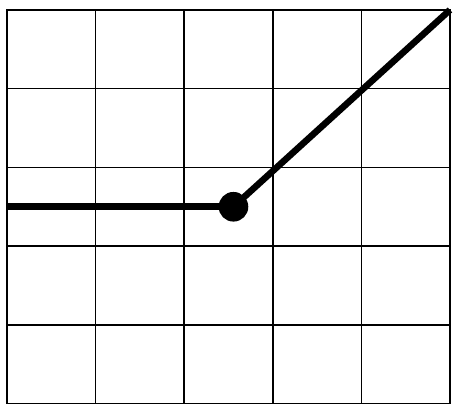} &
\includegraphics[width=20mm]{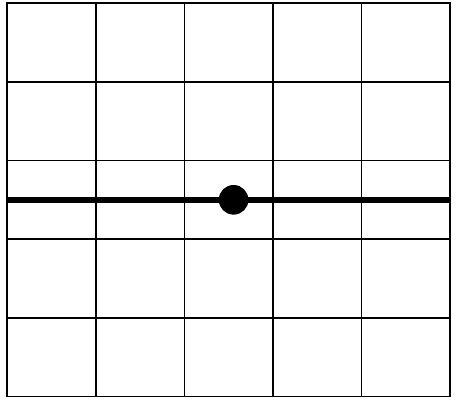} &
\includegraphics[width=20mm]{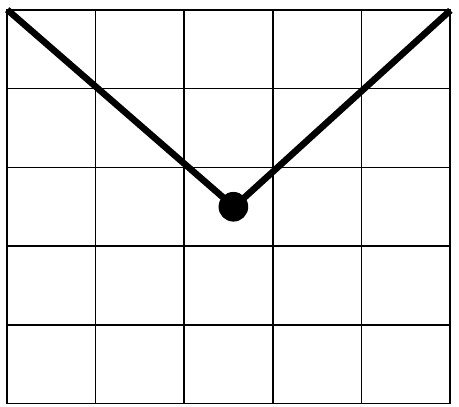} &
\includegraphics[width=20mm]{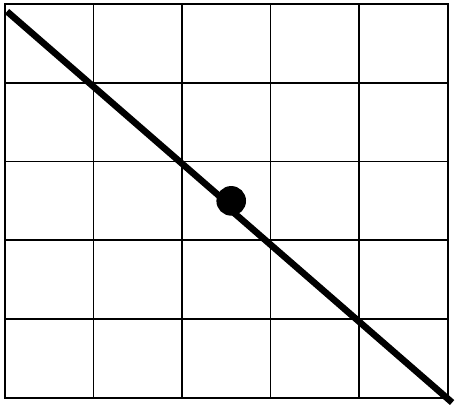} &
\includegraphics[width=20mm]{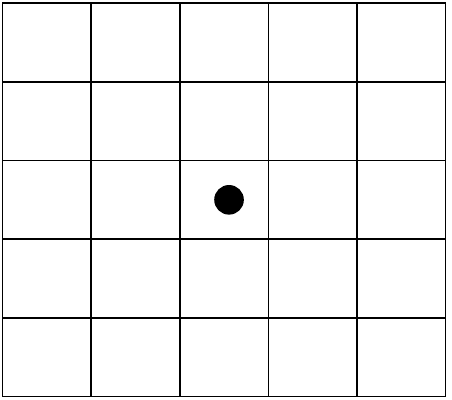} \\
$\frac{3 \pi}{4}$ & $\frac{\pi}{2}$ & $\frac{\pi}{4}$ & $0$ & $\frac{\pi}{2}$ & $0$ & $0$
\end{tabular}
\end{center}
\caption{Iterpretation of the segmentation boundary and it curvature penalties fore each of the windows in Figure \ref{basewindows}.}
\label{windowinterp}
\end{figure*}

In this appendix we describe our approach of determining patch costs for $\pi/4$ precision curvature with patches of size $3\times3$
(the case of $\pi/8$ precision with $5\times5$ patches is similar). Since patches are overlapping changes in boundary direction will be visible in the assignments of more than one super node. We need to make sure that the total contribution of the patch assignments equals the curvature of the segmentation boundary.

To determine the assignments in the case of $3\times 3$ patches, we generate windows of size $5 \times 5$.
These windows contain the binary assignments that would result from a segmentation boundary which transitions between two directions at the center of the window, see Figures~\ref{basewindows} and \ref{windowinterp}. 
Note that all combinations of edges are obtained through symmetries (rotations, reflections and inversions) of these windows.
By looking in these windows we can determine all super node assignments that are present in the vicinity of such a transition and constrain their sum to be the correct curvature penalty.

\begin{figure}[htb]
\begin{center}
\begin{pgfpicture}{0cm}{0cm}{8cm}{5cm}
\pgfputat{\pgfxy(0,0)}{\includegraphics[width=80mm]{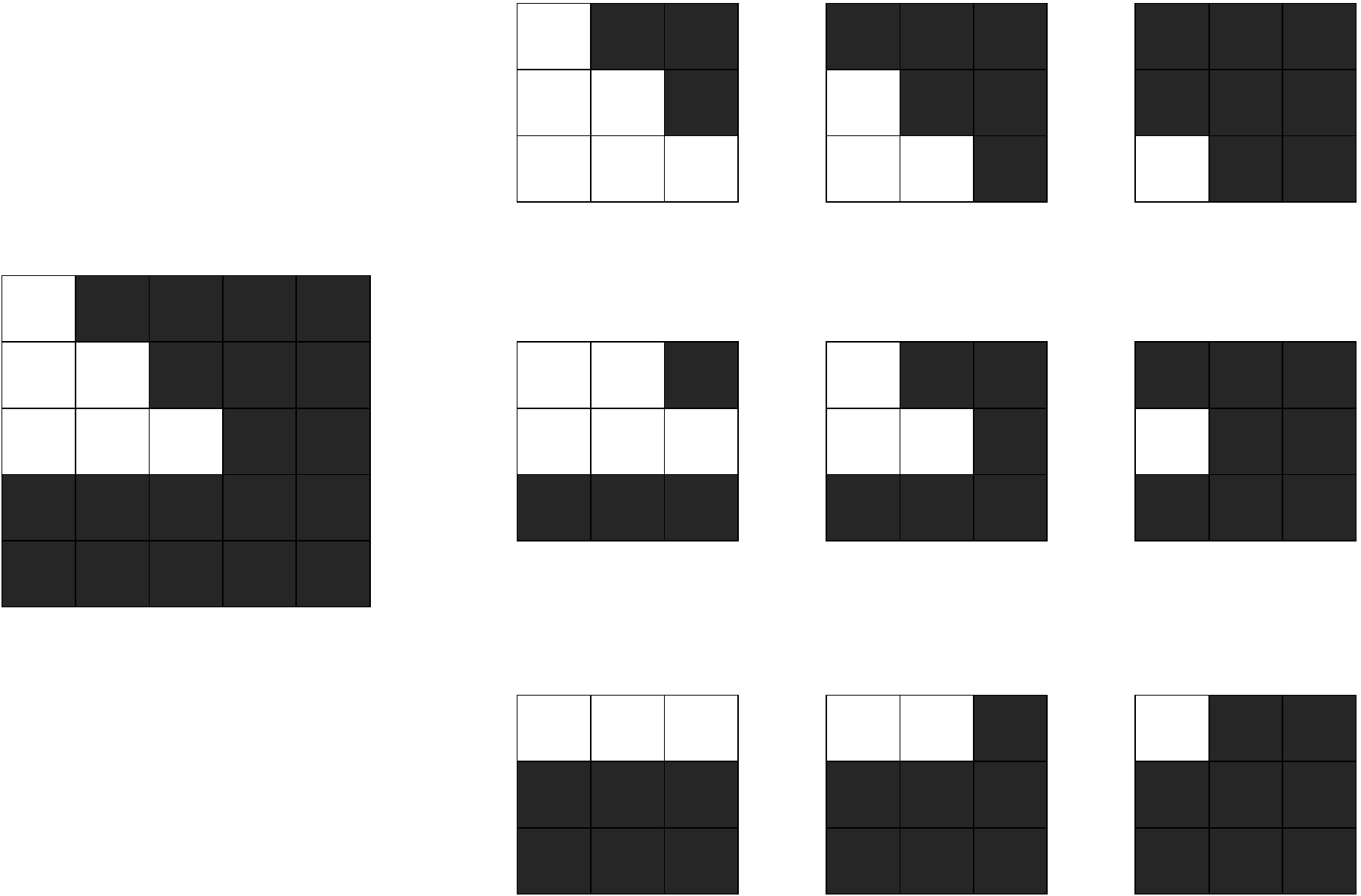}}
\pgfputat{\pgfxy(3.5,5.4)}{$l_{38}$}
\pgfputat{\pgfxy(5.3,5.4)}{$l_{311}$}
\pgfputat{\pgfxy(7.2,5.4)}{$l_{447}$}
\pgfputat{\pgfxy(3.5,3.4)}{$l_{452}$}
\pgfputat{\pgfxy(5.3,3.4)}{$l_{486}$}
\pgfputat{\pgfxy(7.2,3.4)}{$l_{503}$}
\pgfputat{\pgfxy(3.5,1.3)}{$l_{504}$}
\pgfputat{\pgfxy(5.3,1.3)}{$l_{508}$}
\pgfputat{\pgfxy(7.2,1.3)}{$l_{510}$}
\end{pgfpicture}
\end{center}
\caption{First window of Figure~\ref{basewindows} and its super node assignments.}
\label{windowex}
\end{figure}

Consider for example the window in Figure~\ref{windowex}. If we let the labels of the super nodes be $l_a$ where $a \in {0,...,511}$ encodes the state of the individual pixels then the patch in Figure~\ref{windowex} gives us the linear constraint
\begin{eqnarray}
l_{38}+l_{311}+l_{447}+l_{452}+l_{486}+l_{503} & & \nonumber\\ +l_{504}+l_{508}+l_{510} & = & \frac{3\pi}{4}.
\end{eqnarray}
In a similar way each window/boundary transition gives us a linear equality constraint. In addition we require that $l_a \geq 0$, $\forall a \in {0,...,511}$
and that $l_a = 0$ for the labels that do not occur in any of the windows.
This gives us a system of linear equalities and inequalities. To find a solution we randomly select a linear cost function (with positive entries) and solve the resulting linear program. Since the system is under determined the individual label costs can vary depending on the random objective function. 
However, the linear equalities ensure that the resulting curvature estimate obtained for each transition between boundary directions is correct when combining patch assignments.

\end{document}